%% file: example_paper.tex
\documentclass{article}

\usepackage{lipsum} 
\usepackage{microtype}
\usepackage{graphicx}
\usepackage{amssymb}
\usepackage{subcaption}
\usepackage[most]{tcolorbox}
\usepackage{booktabs} 
\input{packages_icml}

\usepackage{hyperref}

\newcommand{\midsize}{\fontsize{7.5pt}{9pt}\selectfont}

\usepackage[accepted]{icml2026}

\usepackage{amsmath}
\usepackage{amssymb}
\usepackage{mathtools}
\usepackage{amsthm}

\usepackage[capitalize,noabbrev]{cleveref}

\theoremstyle{plain}

\usepackage[textsize=tiny]{todonotes}

\icmltitlerunning{Generalized Discrete Diffusion with Self-Correction}

\begin{document}

\twocolumn[
  \icmltitle{Generalized Discrete Diffusion with Self-Correction}



  \icmlsetsymbol{equal}{*}
  \begin{icmlauthorlist}
    \icmlauthor{Linxuan Wang}{equal,Purdue}
    \icmlauthor{Ziyi Wang}{equal,Purdue}
    \icmlauthor{Yikun Bai}{equal,Purdue}
    \icmlauthor{Wei Deng}{ms}
    \icmlauthor{Guang Lin$^\dagger$}{Purdue}
    \icmlauthor{Qifan Song$^\dagger$}{Purdue}
  \end{icmlauthorlist}
  
  \icmlaffiliation{Purdue}{Purdue University}
  \icmlaffiliation{ms}{Morgan Stanley}

  \icmlcorrespondingauthor{Wei Deng}{weideng056@gmail.com}
  \icmlkeywords{Machine Learning, ICML}

  \vskip 0.3in
]



\printAffiliationsAndNotice{\icmlEqualContribution $^\dagger$Equal advising}

\begin{abstract}
    Self-correction is an effective technique for maintaining parallel sampling in discrete diffusion models with minimal performance degradation. Prior work has explored self-correction at inference time or during post-training; however, such approaches often suffer from limited generalization and may impair reasoning. GIDD \cite{rutte2025generalized} pioneers pretraining-based self-correction via a multi-step BERT-style uniform-absorbing objective. However, GIDD relies on a continuous interpolation-based pipeline with opaque interactions between uniform transitions and absorbing masks, which complicates hyperparameter tuning and hinders practical performance. In this work, we propose a \textbf{S}elf-\textbf{C}orrecting \textbf{D}iscrete \textbf{D}iffusion (SCDD) model to reformulate pretrained self-correction with explicit state transitions and learn directly in discrete time. Our framework also simplifies the training noise schedule, eliminates a redundant remasking step, and relies exclusively on uniform transitions to learn self-correction. Experiments at GPT-2 scale show that our method enables more efficient parallel decoding while preserving generation quality. Our code is available at: {\small\url{https://github.com/laaaarrywang/Self-Correcting-Discrete-Diffusion.git}}.
    
\end{abstract}

\input{intro}

\textbf{Notation.}\quad Let $\mathbf{m}$ denote the one-hot vector representing the \textsc{[mask]} token, and let $\mathbf{1}_p$ denote a $p$-dimensional vector of ones. Define the one-hot categorical vocabulary (including the \textsc{[mask]} token) as
\[
\mathcal{V} := \left\{\mathbf{x} \in \{0,1\}^{K+1} : \sum_{i=1}^{K+1} \mathbf{x}_i = 1_{K+1} \right\},
\]
with $|\mathcal{V}| = K+1$. We further define $\mathcal{V}^{/\mathbf{m}} := \mathcal{V} \setminus\{\mathbf{m}\}$ as the vocabulary excluding the \textsc{[mask]} token. We use $\mathrm{Cat}(\cdot\,;\boldsymbol{\pi})$ to denote the categorical distribution over $\mathcal{V}$ with probability vector $\boldsymbol{\pi} \in \Delta^{K+1}$, where $\Delta^{K+1}$ is the $(K+1)$-simplex.

\section{Preliminaries}
Masked diffusion models (MDM/MDLM) \cite{sahoo2024simple,shi2024simplified} are a class of generative models that corrupt data by gradually replacing clean tokens with a special \textsc{[mask]} token. 
The forward process starts from a clean token $\mathbf x \in \mathcal{V}^{/\mathbf{m}} $ and progressively increases the probability of transitioning to $\mathbf m$. Formally, for time points $s,t$ with $0 \le s < t \le 1$, the forward kernel is defined as
$$
q(\mathbf z_t \mid \mathbf z_s)
=
\text{Cat}\left(\mathbf z_t;\alpha_{t|s}\, \mathbf z_s
+
(1-\alpha_{t|s})\, \mathbf m\right),
$$
where $\alpha_{t|s} = \alpha_t / \alpha_s$ and $\alpha_t \in [0,1]$ is a noise scheduling function controlling the survival rate of the token. The induced marginal is
$$
q(\mathbf z_t \mid \mathbf x) = \text{Cat}\left(\mathbf z_t;\alpha_t \mathbf x + (1 - \alpha_t)\, \mathbf m\right).
$$
As $t \rightarrow 1$, the process collapses to the prior $q_1(\mathbf z_1) = \mathbf 1_{\{\mathbf z_1=\mathbf m\}}$, meaning that the clean token is replaced by the \textsc{[mask]} token.

Given this forward process, sampling proceeds in the reverse direction using a learned parametric kernel $p_\theta(\mathbf z_s \mid \mathbf z_t)$ that denoises \textsc{[mask]} tokens back to clean tokens. The denoiser is trained by minimizing the negative ELBO, which reduces to
$$
\mathcal{L}
=
\mathbb{E}_{t,\mathbf x,\,\mathbf z_t}
\left[
\frac{\alpha_t'}{1-\alpha_t}
\log \left(\mathbf x_\theta(\mathbf z_t,t)^\top \mathbf x\right)
\right],
$$
where $\mathbf x_\theta(\mathbf z_t,t)$ is a denoiser that predicts the clean token distribution.

\textbf{Limitations.}\quad Despite its scalable training, MDLM lacks an explicit self-correction mechanism to revise low-quality tokens from earlier steps, degrading the quality of parallel generation \citep{jiang2025diffusion}. Consequently, MDLM typically decodes only a few tokens per step to achieve competitive reasoning performance \citep{zhao2025d1, rojas2026improving}, often making it even slower than autoregressive (AR) alternatives with speculative decoding. 

%


\section{Self-Correcting Discrete Diffusion Model}\label{sec:scdd}

To maximize parallel generation potential, we reformulate the MDLM pipeline to enhance self-correction during pre-training, a stage known to yield strong generalization performance \citep{SutskeverDwarkesh2025ScalingToResearch}. Motivated by BERT’s success in semantic understanding \citep{devlin2019bert}, we incorporate uniform transitions alongside the absorbing-mask in the forward process. While related ideas have been explored in prior works, such as GIDD \citep{rutte2025generalized}\footnote{We refer readers to Appendix Section~\ref{sec:related_work} for further discussion.}, our approach differs by introducing clear and explicit state transitions and eliminating the need for an additional remasking step completely. This simplification substantially eases hyperparameter tuning and improves parallel generation through enhanced self-correction.



\subsection{Forward Noising Process}
\label{sec:2.1}
For clarity, we first derive the model for the single token case before generalizing it to the sequence case. We consider a discrete-time forward Markov noising process $q$ that transforms the clean data $\mathbf x$ into latent variables $\{\mathbf z_{t_i}\}_{i=0}^T$. A typical choice of time grid $\{t_i\}_{i=0}^T$ is equal-spaced grid $\{\frac{i}{T}\}_{i=0}^T$, which is also the one that we use throughout this paper. For notation convenience, we define $\mathbf z_{t_{-1}} = \mathbf x$ and interchangeably use them to refer to the clean data. The joint distribution of clean data and latent variables follows
\begin{align}
    q(\mathbf x,\mathbf z_0,...,\mathbf  z_1) = q_{\text{data}}(\mathbf x)\prod_{i=0}^T q(\mathbf z_{t_i}\mid\mathbf z_{t_{i-1}}),
    \label{eq:jointDist}
\end{align}
where $q_{\text{data}}$ is the unknown data distribution.
Let $\mathbf u = \tfrac{1}{K}(\mathbf 1_{K+1} - \mathbf m)$ denote the probability vector for the uniform distribution over all non-\textsc{[mask]} tokens. We use $\mathbf u$ to refer to both the vector and the distribution itself. To elicit self-correction capability, we are interested in forward processes whose marginal distributions of $\mathbf z_t$ takes the form
\begin{align}
q(\mathbf z_t\mid\mathbf x)=\text{Cat}\Bigl(\mathbf z_t;\gamma_t \bigl(\rho_t \mathbf x+(1-\rho_t)\mathbf u\bigr)+(1-\gamma_t)\mathbf{m}\Bigr),\label{eq:marginal}
\end{align}
where $\rho_t,\gamma_t:[0,1]\to[0,1]$ with boundary conditions $\rho_1 = \gamma_1= 0$ and $\rho_0, \gamma_0 \approx 1$. Moreover, $\rho_0,\gamma_0 \to 1^-$ as $T\to\infty$. The mixture of $\mathbf x$, $\mathbf u$ and $\mathbf m$ in the probability vector ensures the diffusion model encounters both uniform noise and \textsc{[mask]} tokens during training. At time $t=0$, the latent variable $\mathbf z_0$ is slightly corrupted relative to $\mathbf x$; however, this corruption probability vanishes as the number of timesteps $T$ approaches infinity. At time $t=1$, the clean data is completely corrupted and becomes \textsc{[mask]}. For intermediate times $t \in (0, 1)$, the token can be replaced by another non-\textsc{[mask]} token or masked.

We refer to $\rho_t\gamma_t$, the probability mixture ratio of $\mathbf x$, as the probability that ``$\mathbf z_t$ \emph{retains} $\mathbf x$'', meaning that the event $\{\mathbf z_t = \mathbf x\}$ occurs due to an explicit stay-in mechanism rather than by coincidence from uniform sampling. 
In particular, this excludes the case where $\mathbf z_t$ is sampled from $\mathbf u$ and happens to equal $\mathbf x$. Similarly, $\gamma_t(1-\rho_t)$ is referred to as the probability that ``$\mathbf z_t$ is sampled from $\mathbf u$''. The parameters can also be individually interpreted as two kinds of signal-to-noise ratios (SNR):
\begin{align}
\gamma_t=q(\mathbf z_t\neq\mathbf{m}\mid\mathbf x),\quad\rho_t=q(\mathbf z_t\text{ retains }\mathbf x\mid\mathbf z_t\neq\mathbf{m},\mathbf x)\nonumber,
\end{align}
where $\gamma_t$ measures the SNR of absorbing mask, and $\rho_t$ measures the SNR of uniform transitions. For notation convenience, we set $\rho_{-1} = \gamma_{-1} = 1$.

There are infinite forward noising processes with marginals agree with \eqref{eq:marginal}. However, to eliminate the remasking behavior during token generation (i.e., the backward process), the \textsc{[mask]} state $\mathbf m$ needs to be an absorbing state of the forward process. Therefore, we make an additional assumption that $\rho_t, \gamma_t$ are monotonically decreasing sequences. Under this mild assumption, the marginal in \eqref{eq:marginal} can be realized by a Markov chain of which $\mathbf{m}$ is an absorbing state, similar to masked diffusion models \cite{sahoo2024simple,shi2024simplified}. 
\begin{proposition}
When $\rho_t, \gamma_t$ are monotonically decreasing, the following forward Markov transition kernel induces the marginal distribution \eqref{eq:marginal}:
\begin{align}
q&(\mathbf z_t\mid\mathbf z_s):= \nonumber \\ 
&\begin{cases}
\gamma_{t|s}\left(\mathbf 1_{\{\mathbf z_s=\mathbf z_t\}}\rho_{t|s}+\tfrac{1-\rho_{t|s}}{K}\right),& \text{if }\mathbf z_t,\mathbf z_s \neq\mathbf m,\\
1-\gamma_{t|s}, & \text{if }\mathbf z_t=\mathbf m, \mathbf z_s\neq\mathbf m,\\
1, & \text{if } \mathbf z_t = \mathbf z_s=\mathbf{m},
\end{cases}\label{eq:forward}
\end{align}

where $t,s \in \{t_{-1},t_0,...,t_T\}$ are two adjacent time points satisfying $t>s$, $\rho_{t|s}:=\tfrac{\rho_t}{\rho_s}$, and $\gamma_{t|s}:=\tfrac{\gamma_t}{\gamma_s}$.
\end{proposition}
\vspace{-0.5cm}
\begin{proof}
Since $\rho_t$ and $\gamma_t$ are monotonically decreasing, the above transition kernel is well-defined. The Markovian forward process implied by the forward kernel \eqref{eq:forward} can be equivalently characterized as
\begin{equation}
    q(\mathbf z_t\mid\mathbf z_s) = \text{Cat}\left(\mathbf z_{t}; \mathbf z_s^\top Q_{t|s}\right)\nonumber 
\end{equation}
where
\begin{equation}
    Q_{t|s} =\left(\begin{array}{cc}
\rho_{t|s}\gamma_{t|s} \mathbf I+\frac{(1-\rho_{t|s})\gamma_{t|s}}{K}\mathbf 1\mathbf 1^\top &  (1-\gamma_{t|s})\mathbf 1 \\[2pt]
\mathbf{0}^{\top} & 1
\end{array}\right)
    \label{eq:forwardMatrix}
\end{equation}
is the forward transition matrix, where $\mathbf I$ and $\mathbf 1$ are abbreviations of the $K$-dimensional identity matrix and $K$-dimensional all-ones vector. Let $i$ be the unique index such that $t_i=t$, the marginal distribution of $\mathbf z_t$ is therefore
\begin{equation}
    q(\mathbf z_t\mid\mathbf x) = \text{Cat}(\mathbf z_t; \mathbf x^\top \bar{Q}_t),
    \label{eq:forwardMarginal}
\end{equation}
where
\begin{equation}
\begin{aligned}
    \bar{Q}_t &= \prod_{j=0}^iQ_{t_j|t_{j-1}}\\
    &=\left(\begin{array}{cc}
\bar{A}_t & (1-\prod_{j=0}^i \gamma_{t_j|t_{j-1}} ) \mathbf{1} \\
\mathbf{0}^{\top} & 1
\end{array}\right)\\
&=\left(\begin{array}{cc}
\bar{A}_t & (1-\gamma_t) \mathbf{1} \\
\mathbf{0}^{\top} & 1
\end{array}\right).
\end{aligned} \nonumber 
\end{equation}
$\bar{A}_t$ is greatly simplified due to the idempotence \footnote{Define \(\mathbf P:=\mathbf 1\mathbf 1^\top /K\). Since \(\mathbf P^2=\mathbf P\), for any scalars \(a,b\), we have
$(a\mathbf I+(1-a)\mathbf P)(b\mathbf I+(1-b)\mathbf P)
=ab\mathbf I+(1-ab)\mathbf P$.

}
\begin{equation}
\begin{aligned}
    \bar{A}_t &= \prod_{j=0}^i \Bigl(\rho_{t_j|t_{j-1}}\gamma_{t_j|t_{j-1}} \mathbf I+\tfrac{(1-\rho_{t_j|t_{j-1}})\gamma_{t_j|t_{j-1}}}{K}\mathbf 1\mathbf 1^\top\Bigl) \\
    & = \Bigl(\prod_{j=0}^i \rho_{t_j|t_{j-1}}\gamma_{t_j|t_{j-1}}\Bigr) \mathbf  I  \\
    &+\frac{1}{K}\Bigl(\prod_{j=0}^i\gamma_{t_j|t_{j-1}} -\prod_{j=0}^i \rho_{t_j|t_{j-1}}\gamma_{t_j|t_{j-1}}\Bigr) \mathbf{1} \mathbf {1}^{\top} \\
    &= \gamma_t\left(\rho_t I + \tfrac{1}{K}(1 - \rho_t)\mathbf 1\mathbf 1^\top\right),
\end{aligned}\nonumber 
\end{equation}
Therefore, substituting the closed-form expression of $\bar Q_t$ into
\eqref{eq:forwardMarginal}, we obtain that, for any ordinary token
$\mathbf z_t\neq \mathbf m$,
\[
q(\mathbf z_t\mid \mathbf x)
=
\gamma_t\left(
\rho_t \mathbf 1_{\{\mathbf z_t=\mathbf x\}}
+
\frac{1-\rho_t}{K}
\right),
\]
while
\[
q(\mathbf z_t=\mathbf m\mid \mathbf x)=1-\gamma_t.
\]
This is exactly the marginal distribution specified in \eqref{eq:marginal}.
\end{proof}

\paragraph{Forward Process in Continuous Time.} The discrete-time Markov chain introduced above can be extended to continuous time via a continuous-time Markov chain (CTMC) model. Although deriving SCDD doesn't require CTMC theory, we include this analysis for the sake of completeness. Moreover, when we borrow the forward transition rate concept, we obtain a better interpretation of the noise schedules of our model. As before, we assume that $\gamma_t$ and $\rho_t$ are differentiable, decreasing functions of time $t\in[0,1]$. 

Let $s$ and $t$ be two adjacent time points with $\Delta t = t-s > 0$. We define the infinitesimal generator $R_t(\cdot,\cdot)$ by
\begin{align}
R_t(\mathbf z_t,\mathbf z_s)
:=
\begin{cases}
\left(\frac{\gamma_t'}{\gamma_t} + \frac{\rho'_t}{\rho_t}\right) \mathbf z_t^\top \mathbf z_s \\
\ - \mathbf z_t^\top\Bigl(
\tfrac{\rho'_t}{\rho_t}\mathbf{u}
+\tfrac{\gamma_t'}{\gamma_t}\mathbf{m}
\Bigr), &\mathbf z_s\neq \mathbf m\\
0, &\mathbf z_s=\mathbf m.
\end{cases}
\label{eq:forward_rate}
\end{align}
The following lemma proves that \eqref{eq:forward_rate} is the forward transition rate of the process \eqref{eq:forward}. We leave the complete proof to Appendix \ref{sec:forward-rate-proof}.
\begin{lemma}\label{lem:forward-rate}
Let $R_t$ be defined as in \eqref{eq:forward_rate}. Then, $R_t$ is a valid infinitesimal generator (rate matrix) in the sense of \citet{lipman2024flow, gat2024discrete}, i.e.,
$$
R_t(\mathbf z_t,\mathbf z_s)\ge 0,\forall \mathbf z_t\neq \mathbf z_s,
\ \sum_{\mathbf z_t} R_t(\mathbf z_t,\mathbf z_s)=0.
$$
Moreover, the discrete-time forward kernel \eqref{eq:forward} admits the continuous-time expansion:
$$
q(\mathbf z_t \mid \mathbf z_s)=\delta_{\mathbf z_t,\mathbf z_s}+\Delta t\,R_t(\mathbf z_t,\mathbf z_s)+o(\Delta t).
$$
Therefore, the probability evolution induced by $R_t$ is the continuous-time limit of the discrete Markov chain \eqref{eq:forward}.
\end{lemma}

\begin{remark}
When $\mathbf z_s \neq \mathbf m$, $\rho_t$ and $\gamma_t$ govern the forward transition rates for sampling $\mathbf z_t$ from $\mathbf u$ and setting $\mathbf z_t$ to \textsc{[mask]}, respectively. Thus, in addition to their SNR interpretation in the marginals, they act as explicit and independent controllers of the two noise mechanisms. In contrast, under the same marginal representation, GIDD couples the absorbing-mask and uniform-transition components in its forward rates, preventing independent control without compromising the clean marginal form. SCDD preserves this marginal structure while decoupling the two transition rates. See Table \ref{tab:generator_comparison} for a generator-level comparison and Appendices \ref{correction_mdlm} and \ref{sec:scdd-gidd} for further discussion.
\end{remark}
\subsection{Backward Denoising Process}
Follow the same notations, the true posterior $q(\mathbf z_s \mid \mathbf z_t, \mathbf x)$ derived from Bayes' rule is given by: 
\begin{align}
&q(\mathbf z_{s}\mid \mathbf z_{t},\mathbf x)=\nonumber\\
&\begin{cases}
\frac{\mathbf \rho_s \mathbf z_s^\top \mathbf x+(1-\rho_s)\frac{1}{K}}{\rho_t \mathbf z_t^\top \mathbf x+(1-\rho_t)\frac{1}{K}}\bigl(\frac{\rho_t}{\rho_s}\mathbf z_s^\top \mathbf z_t+\frac{\rho_s-\rho_{t}}{\rho_s}\frac{1}{K}\bigr), &\mathbf z_t\neq\mathbf m\\[6pt]
(1-\mathbf z_{s}^\top \mathbf m) \frac{\gamma_s-\gamma_t}{1-\gamma_t}\left(\rho_s \mathbf z_s^\top \mathbf x+(1-\rho_s)\tfrac{1}{K}\right)\\[2pt]
\quad+\mathbf z_s^\top\mathbf m\frac{1-\gamma_s}{1-\gamma_t},
 &\mathbf z_t=\mathbf{m}
\end{cases}\label{eq:backward}
\end{align}
We refer to the Appendix \ref{sec:backward} for the details of the derivation. 
\vskip 0.2in
\paragraph{Efficient Parallel Self-Correction.} Since we have \textsc{[mask]} being as an absorbing state in the forward process, it follows that there is no transition from non-\textsc{[mask]} tokens to \textsc{[mask]}, i.e., no remasking, in the backward process. By eliminating the redundant remasking step, SCDD acquires more correction capacity for any given number of inference steps. In particular, SCDD can be twice as efficient as purely remasking-based self-correcting discrete diffusion models, since it only needs one step to correct a token instead of two. As a result, it potentially improves the generation quality, especially in few-step generation scenarios.

Since we don't know the true clean data $\mathbf x$, we follow previous work \cite{sohl2015deep, sahoo2024simple} to parameterize the backward process as
\begin{align}&p_\theta(\mathbf z_s\mid\mathbf z_t):=q(\mathbf z_{s}\mid \mathbf z_{t},\mathbf x=\mathbf x_\theta(\mathbf z_t,t))\nonumber\\
&=\begin{cases}
\frac{\mathbf \rho_s \mathbf z_s^\top \mathbf x_\theta+(1-\rho_s)\frac{1}{K}}{\rho_t \mathbf z_t^\top \mathbf x_\theta+(1-\rho_t)\frac{1}{K}}\bigl(\frac{\rho_t}{\rho_s}\mathbf z_s^\top \mathbf z_t+\frac{\rho_s-\rho_{t}}{\rho_s}\frac{1}{K}\bigr), &\mathbf z_t\neq\mathbf m,\\[6pt]
(1-\mathbf z_{s}^\top \mathbf m) \frac{\gamma_s-\gamma_t}{1-\gamma_t}\left(\rho_s \mathbf z_s^\top \mathbf x_\theta+(1-\rho_s)\tfrac{1}{K}\right)\\\quad+\mathbf z_s^\top\mathbf m\frac{1-\gamma_s}{1-\gamma_t},
 &\mathbf z_t=\mathbf{m},
\end{cases}\label{eq:backwardParam}\end{align}
where
$$
\begin{aligned}
\mathbf x_\theta(\cdot,\cdot):
\mathcal{V}\times[0,1]&\to\Delta^{K+1},\nonumber \\
(\mathbf z_t,t)\ &\mapsto\mathbf x_\theta(\mathbf z_t,t),\nonumber
\end{aligned}
$$
is a denoising neural network that predicts clean data $\mathbf x$ given the latent variable $\mathbf z_t$ at time $t$. We omit the arguments of $\mathbf x_\theta$ when no confusion arises. Such parameterization indeed defines a valid density $p_\theta(\cdot\mid\mathbf z_t)$ conditional $\mathbf z_t$, see Appendix \ref{sec:model} for more details. We also derive the corresponding backward transition rate under CTMC theory in Appendix \ref{sec:backwardrate} for completeness.

We further impose the \textit{Zero Masking Probabilities} constraint on the denoising network $\mathbf x_\theta$ as in MDLM \cite{sahoo2024simple}, i,e., $$\mathbf x_\theta(\mathbf z_t,t)^\top \mathbf m = 0,\quad \forall\ \mathbf z_t, t.$$ However, we release the \textit{Carry-Over Unmasking} constraint in MDLM to allow token self-correction during generation: even if $\mathbf z_t \neq \mathbf m$,  we still allow $\mathbf x_\theta$ to assign nonzero probability mass to tokens other than $\mathbf z_t$ so that previously unmasked tokens can be revised in later denoising steps.

Additionally, we set the prior distribution to be $p_\theta(\mathbf z_1):=\mathbf 1(\mathbf z_1 = \mathbf m)$, indicating that we start the backward process from \textsc{[mask]}. With these parameterizations, we define the generative model of $\mathbf x$ as:
\begin{align}
    p_\theta(\mathbf x)&:= \int p_\theta(\mathbf z_1)p_\theta(\mathbf x\mid\mathbf z_0)\prod_{i=1}^T p_\theta(\mathbf z_{t_{i-1}}\mid\mathbf z_{t_i})\mathrm d\mathbf z_{0:1}\label{eq:p_theta_marginal}
\end{align}



\subsection{ELBO}
To train a diffusion model, we leverage variational inference and minimize the usual negative evidence lower bound (NELBO) as in previous works \cite{sohl2015deep,ho2020denoising, sahoo2024simple}. The discrete-time NELBO for finite $T$ is given by: 
\begin{align}
&\mathcal L^T_{\text{NELBO}}
:=
\underbrace{
-\mathbb{E}_{q} \big[ \log p_\theta(\mathbf x \mid \mathbf z_0) \big]
}_{\mathcal L^T_{\text{reconstruction}}} \nonumber \\
&+
\underbrace{
\mathbb E_{q}\left[D_{\mathrm{KL}}
\big(
q(\mathbf z_1 \mid \mathbf x)
\|
p_\theta(\mathbf z_1)
\big)\right]
}_{\mathcal L_{\text{prior}}}\nonumber\\
&+
\underbrace{
\mathbb E_{q}\left[\sum_{i=1}^T
D_{\mathrm{KL}}
\big(
q(\mathbf z_{t_{i-1}} \mid \mathbf z_{t_i}, \mathbf x)
\|
p_\theta(\mathbf z_{t_{i-1}} \mid \mathbf z_{t_i})
\big)\right]
}_{\mathcal L^T_{\text{diffusion}}},
\label{eq:NELBO}
\end{align}
\begin{figure*}[t]
\begin{align}
\mathcal{L}^T_{\text{diffusion}}=\begin{cases}
-\mathbb E_{t\sim\mathcal U\{t_1,\dots,t_T\}}\mathbb E_{q}\left[ T\sum_{\mathbf v\neq \mathbf m}
\frac{\rho_s\mathbf v^\top \mathbf x
+(1-\rho_s)\frac{1}{K}}{\rho_t \mathbf z_{t}^\top \mathbf x
+(1-\rho_t)\frac{1}{K}}
\Big(\frac{\rho_t}{\rho_s}\mathbf v^\top \mathbf z_{t}+\frac{\rho_s-\rho_t}{\rho_s}\tfrac{1}{K}\Big)
\log\frac{\rho_s\mathbf v^\top \mathbf x_\theta + (1-\rho_s)\frac{1}{K}}
{\rho_t \mathbf z_{t}^\top \mathbf x_\theta +(1-\rho_t)\frac{1}{K}}\right],
&\text{if }\mathbf z_t\neq\mathbf m,\\[4pt]
-\mathbb E_{t\sim\mathcal U\{t_1,\dots,t_T\}}\mathbb E_{q}\left[T\sum_{\mathbf v\neq \mathbf m}\frac{\gamma_{s}-\gamma_t}{1-\gamma_t}(\rho_s\mathbf v^{\top}\mathbf x+(1-\rho_s)\frac{1}{K})\log \left(\rho_s\mathbf v^\top \mathbf x_\theta+(1-\rho_s)\frac{1}{K}\right)\right], &\text{if }\mathbf z_t=\mathbf m.
\end{cases}\label{eq:loss_zt}
\end{align}
\end{figure*}
\begin{figure*}
    \begin{align}
        \mathcal{L}^\infty_{\text{diffusion}}
=
\begin{cases}
\mathbb E_{t\sim\mathcal U[0,1]}\mathbb E_q\Bigl[
\sum_{\mathbf v\neq \mathbf z_t,\mathbf m }\left(
\frac{(\rho_t \mathbf v^\top \mathbf x+(1-\rho_t)\frac{1}{K})\frac{\rho'_t}{\rho_t}\frac{1}{K}}
{\rho_t \mathbf z_t^\top \mathbf x+(1-\rho_t)\frac{1}{K}}
\log
\frac{\rho_t \mathbf v^\top \mathbf x_\theta+(1-\rho_t)\frac{1}{K}}
{\rho_t \mathbf z_t^\top \mathbf x_\theta+(1-\rho_t)\frac{1}{K}}\right)
-
\frac{\rho'_t(-\mathbf z_t^\top \mathbf x_\theta+\frac{1}{K})}
{\rho_t \mathbf z_t^\top \mathbf x_\theta+(1-\rho_t)\frac{1}{K}}\Bigr],
& \text{if } \mathbf z_t\neq \mathbf m,
\\
\mathbb E_{t\sim\mathcal U[0,1]} \mathbb E_q\left[\sum_{\mathbf v\neq\mathbf m}
\frac{\gamma_t'}{1-\gamma_t}
(\rho_t \mathbf v^\top \mathbf x+(1-\rho_t)\frac{1}{K})
\log(\rho_t \mathbf v^\top \mathbf x_\theta+(1-\rho_t)\frac{1}{K})\right],
& \text{if } \mathbf z_t=\mathbf m.
\end{cases}
\label{eq:contLoss}
\end{align}
\end{figure*}
where $q$ is the abbreviation of 
$q(\mathbf x, \mathbf z_{0:1})$ defined in \eqref{eq:jointDist}, and $\mathcal L^T_{\text{diffusion}}$ is given by \eqref{eq:loss_zt} up to $\theta$-independent additive constants. A distinctive feature of our loss function is that the latent state $\mathbf z_t$ contributes to the gradient regardless of whether it is masked ($\mathbf z_t = \mathbf m$) or unmasked ($\mathbf z_t \neq \mathbf m$), which is similar to GIDD \cite{rutte2025generalized}. Note that if $\rho_t\equiv1$, the diffusion loss $\eqref{eq:loss_zt}$ is reduced to MDLM diffusion loss \cite{sahoo2024simple} since the forward noising process doesn't involve uniform noise anymore. We left the derivation of discrete-time ELBO to Appendix \ref{sec:loss}, and the discussion of the relationship with MDLM to Appendix \ref{sec:mdlm-ours}.
\vspace{-0.5cm}
\paragraph{Continuous-time ELBO.} By taking $T\to\infty$, the discrete-time negative ELBO will converge to its continuous-time extension, denoted by $\mathcal L^\infty_{\text{NELBO}}$. This will give tight approximation to ELBO \cite{kingma2021variational, sahoo2024simple}.  
Among the three components in \eqref{eq:NELBO}, $\mathcal L_{\text{prior}} \equiv 0$ by construction. When $T\to\infty$, $\mathcal \lim_{T\to\infty}\mathcal L^T_{\text{reconstruction}}=0$ due to the fact that $\mathbf z_0$ converges in distribution to $\mathbf x$ when $T\to\infty$. Therefore, the only component left in $\mathcal L^\infty_{\text{NELBO}}$ is $\mathcal L^\infty_{\text{diffusion}}:=\lim_{T\to\infty}\mathcal L^T_{\text{diffusion}}$, which is given in \eqref{eq:contLoss} up to $\theta$-independent additive constants. See Appendix \ref{sec:loss} for derivation of continuous-time ELBO.
\vskip 0.2in
\subsection{Generalization to Sequences of Tokens}
In practice, we operate on sequences of tokens rather than single tokens. Our generalization follows \citet{austin2021structured, sahoo2024simple, rutte2025generalized}. Let $\mathbf x^{1:L}$ be a length-$L$ sequence of clean data, and $\mathbf z_t^{1:L}$ be a length-$L$ sequence of latent variables at time $t$. The forward noising process is applied to each of the $L$ positions in the sequence independently to corrupt the clean data $\mathbf x^{1:L}$. The model $\mathbf x_\theta$ now accepts a whole sequence $\mathbf z_t^{1:L}$ as input, and outputs $\mathbf x_\theta(\mathbf z_t^{1:L},t) = (\mathbf x^l_\theta(\mathbf z_t^{1:L},t))_{l=1}^L$, where $\mathbf x^l_\theta(\mathbf z_t^{1:L},t)$ is the predicted distribution of clean data at position $l$. The denoising process admits a factorization over all positions, i.e., $p_\theta(\mathbf z_s^{1:L}\mid\mathbf z_t^{1:L}) = \prod_{l=1}^Lp_\theta\left(\mathbf{z}_s^{\ell} \mid \mathbf{z}_t^{1: L}\right)$. The final training objective is defined as the sum of the per-token NELBOs across all positions.

\subsection{Sampling}
We perform the usual ancestral sampling as in other works in this field \cite{shi2024simplified,sahoo2024simple, rutte2025generalized, wang2025remasking}. To sample from the model $p_\theta(\mathbf x)$, the backward process is initiated from a sequence with \textsc{[mask]} at all positions, and runs iteratively according to $p_\theta(\mathbf z_s^{1:L}\mid\mathbf z_t^{1:L})$ following the prefixed sampling time schedule. Given the current intermediate sample $\mathbf z_t^{1:L}$, the model samples in parallel across all positions in the sequence for the next time step $s<t$ in the schedule. 
\section{Experiments}
\subsection{Experiment Setup}
We focus on applying SCDD to language modeling tasks, and select two standard datasets, LM1B (One Billion Words Dataset, \citet{chelba2013onebillion}) and the OWT (OpenWebText, \citet{Gokaslan2019OpenWeb}), to evaluate SCDD. We use  GPT-2 tokenizer, and adopt a DiT \cite{peebles2023scalable}  backbone at \textsc{small} scale throughout the experiments. The model is trained on 33B tokens for LM1B, and 131B tokens for OWT. 

To ensure fair comparison, we use the noise schedule that induces the same marginal distribution as GIDD during SCDD training. As for baselines, we retrain GIDD \cite{rutte2025generalized} with $p_u \in \{0.1,0.2\}$ and MDLM \cite{sahoo2024simple} under similar settings, where $p_u$ is the maximum uniform transition noise ratio (see Appendix \ref{sec:noise_schedule_scdd}). We also implemented ReMDM-cap and ReMDM-conf \cite{wang2025remasking} to retrained MDLM as two of the baselines. See Appendix \ref{sec:experiments} for training details.
\subsection{Likelihood Evaluation}
\begin{table}[t]
  \caption{Validation perplexity (Val PPL) on LM1B and OWT. $^\dagger$Reported in \citet{sahoo2024simple}. $^*$Total number of tokens seen, calculated as double of the average number of tokens seen as reported in \citet{sahoo2024simple}.}
  \vskip -0.1in
  \label{tab:valppl}
  \begin{center}
      \begin{sc}
        \begin{tabular}{lcc}
          \toprule
  Model (\textsc{small}) & \multicolumn{2}{c}{Val PPL. ($\downarrow$)} \\
  & LM1B & OWT  \\
          \midrule
          $\text{MDLM}^\dagger$(66B$^*$) & 27.04 & - \\
          $\text{MDLM}^\dagger$(524B$^*$) & - &  23.21\\
          \midrule
          MDLM                &  32.98   & 24.72          \\
          GIDD+ ($p_u=0.1$)   &  39.98   & 31.54          \\
          GIDD+ ($p_u=0.2$)   &  40.70   & 32.19          \\
          SCDD ($p_u=0.1$, ours)    &  39.16   & 28.41          \\
          SCDD ($p_u=0.2$, ours)    &  46.54   & 32.49          \\
          \bottomrule
        \end{tabular}
      \end{sc}
  \end{center}
  \vskip -0.2in
\end{table}
Table \ref{tab:valppl} shows the validation perplexity of each model trained on LM1B and OWT dataset. Similar to the findings in \citet{rutte2025generalized}, we see the models with uniform noise added during training exhibit a degradation in validation perplexity due to the increased difficulty of learning transitions between non-\textsc{[mask]} tokens. However, by comparing the best models of SCDD and GIDD, we still see a 3.7\% and 9.9\% decrease in validation perplexity on LM1B and OWT, respectively. Indeed, SCDD is trained without having to learn the transitions from non-\textsc{[mask]} tokens to the \textsc{[mask]} token in the backward process, which slightly eases the training task. We will see later that introducing uniform noise substantially enhances the model’s self-correction capability with this modest increase in validation perplexity.
\subsection{Unconditional Language Generation}
To evaluate the model’s ability to learn from large-scale language data and to assess the effectiveness of self-correction at inference time, we perform unconditional text generation and report the generative perplexity (Gen PPL), a common quality metric for text generation, evaluated by the GPT2-large model in Table \ref{tab:genppl-comparison}. As noted in \citet{zheng2025masked}, Gen PPL can be extremely low if the model produces highly repetitive and redundant tokens. Therefore, we also report the unigram entropy metric in Appendix \ref{sec:addexp} in addition to Gen PPL, serving as a sanity check for the diversity of generated texts. 

SCDD ($p_u=0.2$) consistently outperforms the best GIDD+ baseline across all denoising steps on LM1B and OWT with a comparable entropy. Notably, SCDD has significant improvements over baselines at few-step parallel generation scenarios, with 55\% and 9.2\% decrease in Gen PPL when compared to ReMDM-cap and GIDD+ at 32-step generation, respectively. 
We attribute the performance gain to two primary factors. First, GIDD+ utilizes a dynamic weighting scheme to mitigate training loss explosion. However, this scheme downweights the training samples at boundaries ($t\to0$ and $t\to1$), which may impede the model from learning critical final corrections during generation. Second, GIDD+ does not eliminate remasking during inference, which likely reduces the capacity of self-correction compared to SCDD. 
\begin{table*}[t]

\caption{Generative perplexity (Gen PPL) on LM1B and OWT datasets across sampling steps. Lower is better. Bold values indicate the best performance per column. $^\dagger$We train all models on OWT with a context length of 512 to be consistent with \citet{rutte2025generalized}, different from the 1024 context length in \citet{sahoo2024simple, wang2025remasking}.}
\label{tab:genppl-comparison}
\begin{center}
\begin{sc}
\vskip -0.1in
\renewcommand{\arraystretch}{0.85}
\setlength{\tabcolsep}{5pt}
\begin{tabular}{lccccccccccc}
\toprule
\qquad Gen. PPL ($\downarrow$) & \multicolumn{5}{c}{LM1B (Steps)} & \multicolumn{6}{c}{OWT (Steps)} \\
\cmidrule(r){2-6} \cmidrule(l){7-12}
Model & 16 & 32 & 64 & 128 & 256 & 32 & 64 & 128 & 256 & 512 & 1024 \\
\midrule
MDLM$^\dagger$ & 226.0 & 162.6 & 136.7 & 123.0 & 118.6 & 169.9 & 123.6 & 104.7 & 94.8 & 91.9 & 88.5 \\[3pt]
ReMDM-Cap 0.01 & 222.1 & 157.5 & 127.0 & 108.9 & \textbf{96.8} & 166.3 & 120.9 & 95.9 & 81.7 & 73.9 & 68.3 \\[3pt]
ReMDM-Confidence & 221.1 & 159.5 & 129.8 & 122.8 & 120.4 & 167.6 & 118.3 & 98.1 & 87.9 & 83.9 & 80.5 \\[3pt]
GIDD+ ($p_u=0.1$) & 171.1 & 146.4 & 134.9 & 131.9 & 128.7 & 82.1 & 71.4 & 66.7 & 65.0 & 64.8 & 63.8 \\[3pt]
GIDD+ ($p_u=0.2$) & 192.7 & 165.5 & 151.9 & 147.3 & 144.8 & 90.5 & 79.0 & 75.1 & 73.2 & 72.0 & 71.2 \\[3pt]
SCDD ($p_u=0.1$, ours) & 159.8 & 133.5 & 119.2 & 113.7 & 108.9 & 78.6 & 71.8 & 67.6 & 66.0 & 63.6 & 61.3 \\[3pt]
SCDD ($p_u=0.2$, ours) & \textbf{159.2} & \textbf{130.0} & \textbf{115.2} & \textbf{108.4} & 102.6 & \textbf{74.5} & \textbf{67.1} & \textbf{60.7} & \textbf{59.6} & \textbf{58.2} & \textbf{55.7} \\
\bottomrule
\end{tabular}
\end{sc}
\end{center}
\end{table*}
\paragraph{LLM-as-a-judge Evaluation.} 
To provide stronger semantic evidence showing that the SCDD outputs are actually better, we perform \textit{LLM-as-a-judge} evaluation to directly assess correction quality beyond Gen PPL and entropy. For each setting, we corrupt 256 clean OWT sequences from $t=0$ to $t=0.8$ using SCDD's forward process, then let both models denoise from the identical corrupted input. Each model generates outputs at 6 different total step counts with nucleus sampling ($p=0.9$). The GPT-5.4 judge scores each pair on five dimensions (\textit{Clarity}, \textit{Grammaticality}, \textit{Factuality}, \textit{Style}, \textit{Creativity}) on a 1 to 10 scale. Finally, the model decides the "winning text". 

We report per-metric scores with paired $t$-tests and overall win rates with binomial tests ($n=256$). The evaluation prompt for GPT-5.4 is in Figure \ref{fig:evaluation_prompt}, strictly following \citet{rutte2025generalized}. Results are presented in Table \ref{tab:matched_compact}, showing that SCDD consistently outperforms GIDD+ across \textit{Clarity}, \textit{Factuality}, and \textit{Style} metrics, with significant gains in \textit{Clarity} and \textit{Style} at higher step counts. While GIDD+ exhibits a persistent advantage in Creativity, SCDD's ability to maintain higher overall win rates—reaching 60.6\% at 1024 steps—highlights its self-correction efficacy in matched-noise environments. In addition to this matched-schedule setting, we see similar results in a cross-ratio setting, where we compare the best-performing SCDD ($p_u=0.2$) against best GIDD+ ($p_u=0.1$). See Appendix \ref{sec:addexp} and Table \ref{tab:crossratio_compact} for more details.
\begin{table*}[t]

\caption{Matched-schedule Setting --- SCDD ($p_u{=}0.2$) vs GIDD+ ($p_u{=}0.2$). Values are formatted as SCDD (GIDD+). Significance: $^{*} p < 0.05$, $^{**} p < 0.01$. }
\label{tab:matched_compact}
\begin{center}
\begin{sc}
\renewcommand{\arraystretch}{0.8}
\vskip -0.1in
\setlength{\tabcolsep}{3.5pt} 
\begin{tabular}{lcccccc}
\toprule
Metrics & \multicolumn{6}{c}{Steps} \\
\cmidrule(l){2-7}
 & 32 & 64 & 128 & 256 & 512 & 1024 \\
\midrule
Clarity & 1.65 (1.49)$^{**}$ & 1.62 (1.52) & 1.69 (1.50)$^{**}$ & 1.68 (1.56)$^{*}$ & 1.73 (1.50)$^{**}$ & 1.73 (1.48)$^{**}$ \\[2pt]
Gramm. & 1.38 (1.42) & 1.46 (1.49) & 1.45 (1.46) & 1.49 (1.53) & 1.51 (1.47) & 1.57 (1.46)$^{*}$ \\[2pt]
Fact. & 2.20 (2.11)$^{*}$ & 2.13 (2.07) & 2.20 (2.05)$^{**}$ & 2.21 (2.12)$^{*}$ & 2.13 (1.97)$^{**}$ & 2.20 (2.07)$^{**}$ \\[2pt]
Style & 1.62 (1.50)$^{*}$ & 1.63 (1.53) & 1.66 (1.51)$^{**}$ & 1.66 (1.57) & 1.70 (1.52)$^{**}$ & 1.73 (1.50)$^{**}$ \\[2pt]
Creativity & 2.78 (2.99)$^{**}$ & 2.81 (3.02)$^{**}$ & 2.88 (3.14)$^{**}$ & 2.84 (3.14)$^{**}$ & 2.94 (3.14)$^{**}$ & 2.93 (3.15)$^{**}$ \\[2pt]
\midrule
Win rate & 55.9\% & 53.0\% & 55.3\% & 52.0\% & 58.1\%$^{*}$ & 60.6\%$^{**}$ \\
\bottomrule
\end{tabular}
\vskip -0.1in
\end{sc}
\end{center}
\end{table*}

\paragraph{Correction Capacity.}
To compare the correction capacity between GIDD+ and SCDD, we use the models trained on OWT data, and generate 128 samples for each model under different denoising steps, then calculate the average of \textit{Correction Rate} defined as follows:
$$\textit{Correction Rate} = \frac{C}{L},$$
where $C = \text{\# Total Corrections}$, and $L=\text{\# Context Length}$. Indeed, from Table \ref{tab:correction_rate} we see that SCDD not only achieves significantly higher \textit{Correction Rate}, but also scales faster to a \textit{Correction Rate} of $0.75$ at $1024$ steps, thus leveraging additional denoising steps to refine the generated texts more efficiently.

\begin{table}[h]
  \caption{Comparison of \textit{Correction Rate} between GIDD and SCDD across denoising steps ($N$).}
  \label{tab:correction_rate}
  \begin{center}
    \begin{small}
      \begin{sc}
        \begin{tabular}{lccccc}
          \toprule
          Model ($p_u=0.2$) & \multicolumn{5}{c}{Total \# Denoising Steps $N$ ($\uparrow$)} \\
           & 64 & 128 & 256 & 512 & 1024 \\
          \midrule
          GIDD+  & 0.39 & 0.40 & 0.40 & 0.40 & 0.40 \\
          SCDD (Ours)  & 0.69 & 0.71 & 0.72 & 0.73 & 0.75 \\
          \bottomrule
        \end{tabular}
      \end{sc}
    \end{small}
  \end{center}
\end{table}
\subsection{Ablation Study}
We carry out three ablation studies to investigate into self-correction behaviors. Specifically, we want to know: 1) If a larger uniform noise ratio would encourage more aggressive (parallel) self-correction? 2) How does the timing of peak uniform noise affect the temporal strength of self-correction throughout the generation process? 3) Does “higher correction rate” imply more successful self-corrections, or only more frequent but pointless token revisions?
\paragraph{Parallel Self-Correction.}
\begin{figure}[h]
  \begin{center}
\centerline{\includegraphics[width=\columnwidth]{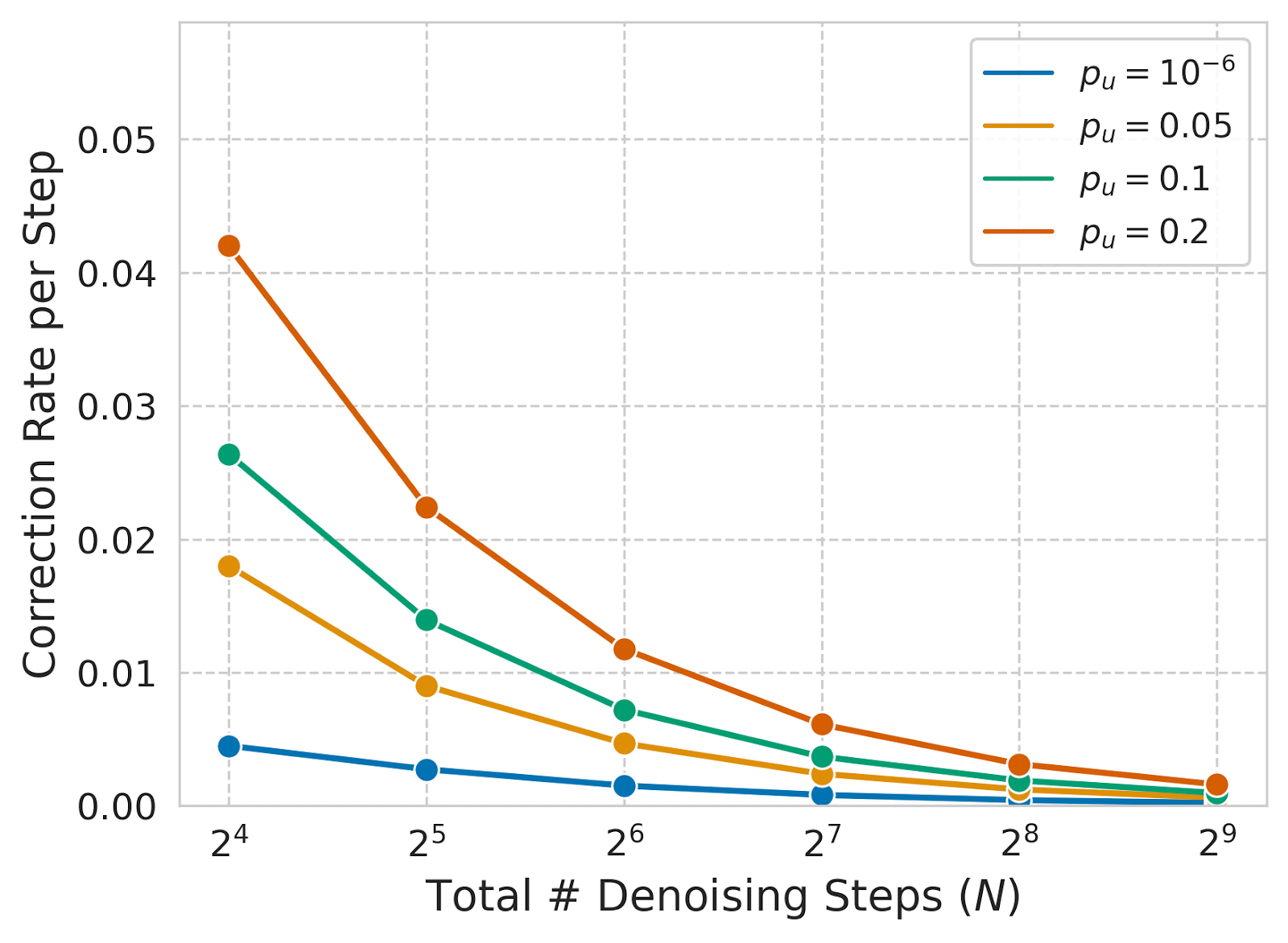}}
    \caption{
      \textit{Correction Rate per Step} versus total number of denoising steps at different maximum uniform noise ratios. Reported values are averaged from 128 independently generated sequences.
    }
    \label{fig:ablation1-1}
  \end{center}
  \vskip -0.1in
\end{figure}
To answer the first question, we train our model on Wikitext-103 \cite{merity2017pointer} with maximum uniform noise ratio $p_u\in\{\text{1e-6},0.05,0.1,0.2\}$, see Appendix \ref{sec:training-details} for training details. Then, we define the following metric:
$$\textit{Correction Rate per Step} = \frac{C}{L\times N},$$
where $C = \text{\# Total Corrections}$, $L=\text{\# Context Length}$, and $N = \text{\# Denoising Steps}$. Note that the \textit{Correction Rate per Step} is derived by applying a normalization factor $N$ to the standard \textit{Correction Rate}. This normalization compensates for the inherent bias toward higher total corrections in generation with more denoising steps as we see in Table \ref{tab:correction_rate}.

For a fixed $N$, Figure \ref{fig:ablation1-1} shows that total corrections increases as $p_u$ increases, consistent with the intuition that more uniform noise in the forward process would incur more corrections in the backward process. Notably, \textit{Correction Rate per Step} is negatively correlated with the number of total denoising steps, indicating more active parallel self-correction at fewer-step generation scenarios. In contrast, self-correction is amortized throughout the whole process in an extended sampling horizon, resulting in a modest increase in total corrections.

\paragraph{Dynamics of Self-Correction.}
\begin{figure}[h]
  \begin{center}
    \centerline{\includegraphics[width=\columnwidth]{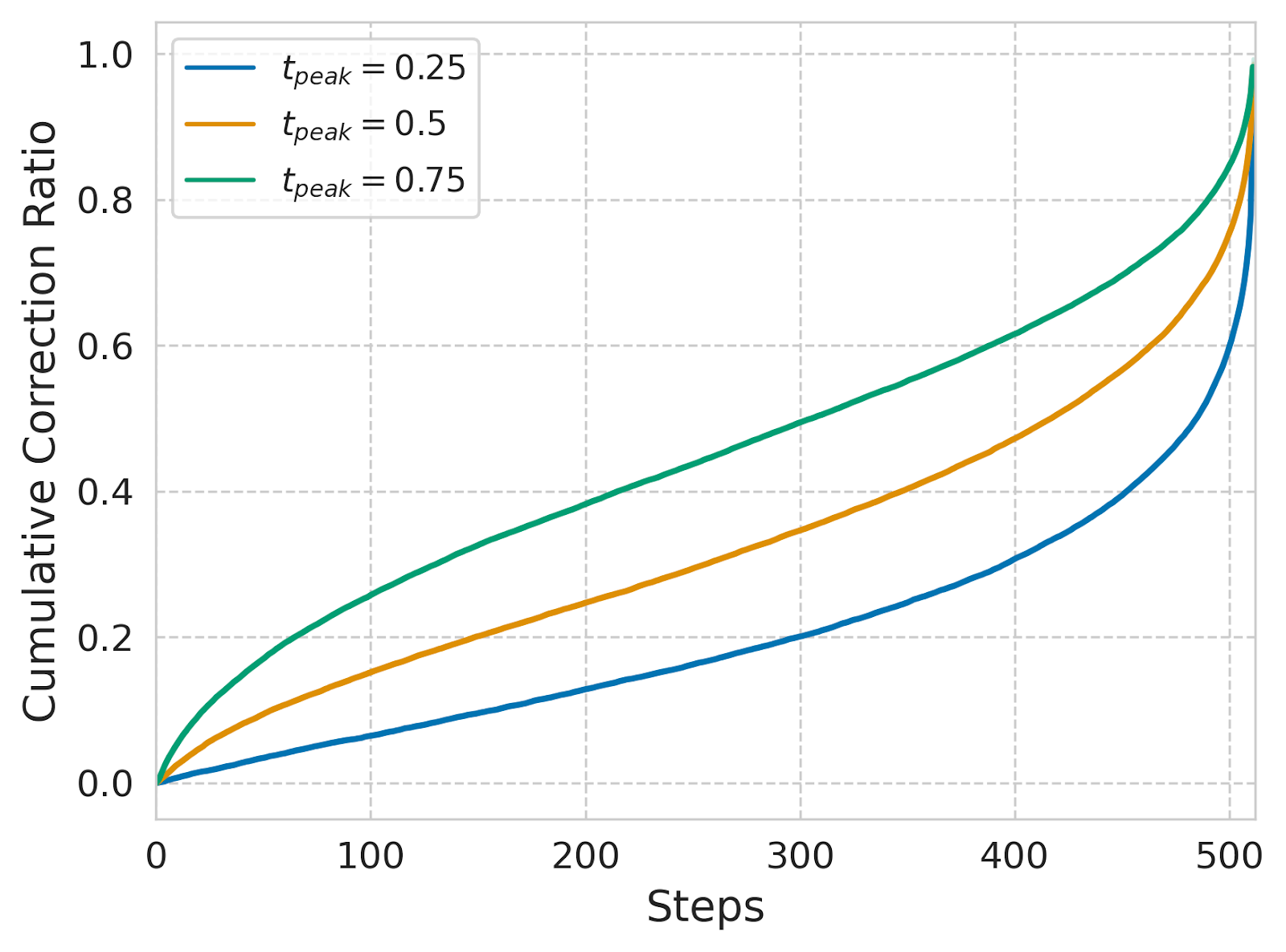}}
    \caption{
        Cumulative correction ratio over 512 denoising steps, defined as the fraction of total corrections completed by step $s$. Results are averaged over 128 independently generated sequences.
    }
    \label{fig:ablation1-2}
  \end{center}
  \vspace{-2em}
\end{figure}

We train the same model on Wikitext-103 data, but with a different noise schedule that attains the maximum uniform noise ratio at a general time point $t_{\text{peak}}$ (see Appendix \ref{sec:noise_schedule_scdd} for a detailed discussion of noise schedules). From Figure \ref{fig:ablation1-2} We see a clear pattern that the model shifts the timing of its self-correction to align with the peak uniform noise time of the training schedule. When the maximum noise ratio occurs later in the forward process (higher $t_{\text{peak}}$), the model tends to self-correct in the early stage during generation, as shown by the concave trajectory of the $t_{\text{peak}} = 0.75$ curve in the first half. In this setting, the model completes nearly 40\% of its total correction within the first 200 steps. Conversely, when the noise peak is shifted to the beginning of the forward process ($t_{\text{peak}} = 0.25$), the correction is significantly delayed, resulting in a convex curve where the majority of the corrections occur only in the final 100 denoising steps. 
\paragraph{Exact Recovery of Corrupted Tokens.}
So far, we have seen that SCDD substantially increases the overall correction rate during generation. To directly verify that these corrections are beneficial rather than spurious revisions, we conduct a controlled corruption--recovery experiment on clean OWT validation sequences. For each sequence, we randomly select \(K\in\{5,10,20,50\}\) positions, replace their tokens with uniformly sampled alternatives, and apply one SCDD (\(p_u=0.2\)) denoising step at the last-step noise level, corresponding to step 127 in a 128-step sampling schedule. We report the \emph{touch rate}, the fraction of corrupted tokens modified by the model, and the \emph{recovery rate}, the fraction of corrupted tokens exactly restored to their original clean values. As shown in Table~\ref{tab:exact-recovery}, SCDD touches nearly all corrupted tokens and exactly recovers \(64.4\%\)--\(69.4\%\) of them. Moreover, under the strongest corruption level \(K=50\), Gen PPL drops from \(154.8\) to \(25.5\) after one denoising step. These results show that SCDD performs meaningful content recovery rather than merely making frequent token edits.

\subsection{Benchmark Performance}
Finally, we evaluate the models on seven standard common sense benchmarks (ARC-E/C \cite{clark2018think}, BoolQ \cite{clark2019boolq}, HellaSwag \cite{zellers2019hellaswag}, PIQA \cite{bisk2020piqa}, OBQA \cite{mihaylov2018can}, Winogrande \cite{sakaguchi2021winogrande}) using the EleutherAI LM Evaluation Harness \cite{eval-harness} with a batch size of 32, see Appendix \ref{sec:ben} for details. Not surprisingly, we see the models trained on uniform and mask noises underperform the mask-only models on all tasks except ARC-c and OBQA. However, it is important to note that these standard benchmarks primarily measure zero-shot likelihood and do not explicitly reflect the self-correction ability observed in our earlier studies. See Appendix \ref{sec:ben} for a detailed discussion of the results.
\begin{table}[t]
\centering
\caption{Recovery rate of corrupted tokens on OWT validation sequences after a single denoising step. Results are averaged over 128 samples with standard errors.}
\label{tab:exact-recovery}
\small
\setlength{\tabcolsep}{3.5pt}
\renewcommand{\arraystretch}{1.35}
\begin{tabular}{ccccc}
\toprule
\(K\) & Touch Rate & Recovery Rate & \multicolumn{2}{c}{Gen PPL} \\
\cmidrule(lr){4-5}
 & &  & Corrupted & Corrected \\
\midrule
5  & \(1.000{\pm}0.000\) & \(0.694{\pm}0.013\) & \(22.0{\pm}0.3\)   & \(23.8{\pm}0.5\) \\
10 & \(1.000{\pm}0.000\) & \(0.652{\pm}0.011\) & \(28.2{\pm}0.4\)   & \(24.0{\pm}0.4\) \\
20 & \(0.999{\pm}0.001\) & \(0.647{\pm}0.008\) & \(44.7{\pm}0.7\)   & \(24.3{\pm}0.5\) \\
50 & \(1.000{\pm}0.000\) & \(0.644{\pm}0.005\) & \(154.8{\pm}2.2\)  & \(25.5{\pm}0.4\) \\
\bottomrule
\end{tabular}
\vspace{-1em}
\end{table}
\section{Conclusion}
In this work, we propose the Self-Correcting Discrete Diffusion (SCDD) model that enhances self-correction to enable more effective parallel generation. In contrast to post-hoc self-corrective samplers, our approach explicitly learns self-correction during pretraining, leading to improved generalization. Compared to GIDD, our forward transition defines clear and explicit state transitions and removes a redundant remasking step, making the model easier to tune and more effective in few-step parallel generation scenarios. Empirically, SCDD demonstrates better generation performance and stronger parallel self-correction capability on LM1B and OWT. For future work, we are interested in scaling up SCDD to acquire more generalizable self-correction ability, as well as exploring reinforcement learning methods to further enhance self-correction ability in reasoning tasks.

\section*{Impact Statement}

This paper presents work whose goal is to advance the field
of Machine Learning. There are many potential societal
consequences of our work, none which we feel must be
specifically highlighted here.

\section*{Acknowledgements}
Guang Lin would like to thank the support of National Science Foundation (DMS-2533878, DMS-2053746, DMS-2134209, ECCS-2328241, CBET-2347401 and OAC-2311848), and U.S.~Department of Energy (DOE) Office of Science Advanced Scientific Computing Research program DE-SC0023161, the SciDAC LEADS Institute, and DOE–Fusion Energy Science, under grant number: DE-SC0024583.




\bibliographystyle{icml2026}
\bibliography{refs_icml}
\newpage 
\appendix 
\onecolumn

\section{Proof of Lemma \ref{lem:forward-rate} }
\label{sec:forward-rate-proof}
\begin{proof}
Firstly, $R_t(\mathbf z_t,\mathbf z_s)\ge 0$ for $\mathbf z_t\neq \mathbf z_s$ and the row sum condition follows by construction, proving that $R_t$ is a valid infinitesimal generator.

Fix $\mathbf z_s \neq \mathbf m$ and consider two adjacent time points $s$ and $t=s+\Delta t$. From the discrete forward kernel \eqref{eq:forward}, we have the following three transition probabilities:
\begin{align}
&q(\mathbf z_t= \mathbf m\mid \mathbf z_s)
=
1-\frac{\gamma_t}{\gamma_s}, \nonumber \\
&q(\mathbf z_t\text{ is sampled from } \mathbf u\mid \mathbf z_s)
=
\frac{\gamma_t}{\gamma_s}\left(1-\frac{\rho_t}{\rho_s}\right),\nonumber 
\\
&q(\mathbf z_t \text{ retains }\mathbf z_s\mid \mathbf z_s)=
\frac{\gamma_t\rho_t}{\gamma_s\rho_s}. \nonumber  
\end{align}
Using first-order Taylor expansions (at $t$) for $\gamma$ and $\rho$,
\begin{align}
\frac{\gamma_t}{\gamma_s} = 1 + \Delta t \frac{\gamma_t'}{\gamma_t} + o(\Delta t),\nonumber\\
\frac{\rho_t}{\rho_s} = 1 + \Delta t \frac{\rho'_t}{\rho_t} + o(\Delta t).\nonumber
\end{align}
Substituting these gives
\begin{align}
&q(\mathbf z_t=\mathbf m\mid \mathbf z_s)
=
\Delta t \left( -\frac{\gamma_t'}{\gamma_t} \right) + o(\Delta t). \nonumber \\
&q(\mathbf z_t \text{ is sampled from } \mathbf u \mid \mathbf z_s)=
\Delta t \left( -\frac{\rho'_t}{\rho_t} \right) + o(\Delta t), \nonumber 
\\
&q(\mathbf z_t \text{ retains }\mathbf z_s\mid \mathbf z_s)
=
1 + \Delta t \left( \frac{\gamma_t'}{\gamma_t} + \frac{\rho'_t}{\rho_t} \right) + o(\Delta t),\nonumber 
\end{align}
These match exactly the off-diagonal entries of $R_t$ and the diagonal term.
Hence
\begin{align}
q(\mathbf z_t\mid \mathbf z_s)
=
\delta_{\mathbf z_t,\mathbf z_s}
+
\Delta t\, R_t(\mathbf z_t,\mathbf z_s)
+
o(\Delta t),\nonumber 
\end{align}
where 
\begin{align*}
R_t(\mathbf z_t,\mathbf z_s)
:=
\begin{cases}
\left(\frac{\gamma_t'}{\gamma_t} + \frac{\rho'_t}{\rho_t}\right) \mathbf z_t^\top \mathbf z_s - \mathbf z_t^\top\Bigl(
\tfrac{\rho'_t}{\rho_t}\mathbf{u}
+\tfrac{\gamma_t'}{\gamma_t}\mathbf{m}
\Bigr), &\mathbf z_s\neq \mathbf m\\
0, &\mathbf z_s=\mathbf m.
\end{cases}
\end{align*}

\end{proof}

\subsection{Forward Rate Matrix}\label{correction_mdlm}
Equivalently, the above forward rate admits a compact matrix form. 
Following mask diffusion models \cite{sahoo2024simple,shi2024simplified}, we adopt the convention that columns represent current states, and all columns sum to zero. Under this convention, the (time-inhomogeneous) generator matrix can be written as
\begin{align}
\mathbf R_t
&=\frac{\gamma_t'}{\gamma_t}\,
\underbrace{\bigl(\mathbf I-\mathbf m\mathbf 1^\top\bigr)}_{\text{MDLM masking structure}}
+\frac{\rho_t'}{\rho_t}\,
\underbrace{\bigl(\mathbf I-\mathbf u\mathbf 1^\top\bigr)\bigl(\mathbf I-\mathbf m\mathbf m^\top\bigr)}_{\text{uniform mixing among non-mask tokens}}.
\label{eq:Rt-matrix-form}
\end{align}
The first term is the MDLM masking generator; the second is the standard uniform generator $\bigl(\mathbf I-\mathbf u\mathbf 1^\top\bigr)$ right-multiplied by $\bigl(\mathbf I-\mathbf m\mathbf m^\top\bigr)$, which zeros out the \textsc{[mask]} column so that \textsc{[mask]} stays absorbing while uniform transitions only shuffle the $K$ non-mask tokens. Both terms have zero column sums, hence so does $\mathbf R_t$.

\begin{proposition}[Consistency of the matrix generator]
\label{pro:Rt-matrix}
Let $\mathbf R_t$ be defined in \eqref{eq:Rt-matrix-form}. Then the forward rate induced by $\mathbf R_t$ via $R_t(\mathbf z_t,\mathbf z_s)=\mathbf z_t^\top\mathbf R_t\mathbf z_s$ coincides with the forward rate \eqref{eq:forward_rate}.
\end{proposition}

\begin{proof}
Using the identity $\mathbf I-\mathbf m\mathbf 1^\top=\bigl(\mathbf I-\mathbf m\mathbf 1^\top\bigr)\bigl(\mathbf I-\mathbf m\mathbf m^\top\bigr)$, we can equivalently write \eqref{eq:Rt-matrix-form} as
\begin{align*}
\mathbf R_t=\left(\Bigl(\frac{\gamma_t'}{\gamma_t}+\frac{\rho_t'}{\rho_t}\Bigr)\mathbf I-\frac{\rho_t'}{\rho_t}\mathbf u\mathbf 1^\top-\frac{\gamma_t'}{\gamma_t}\mathbf m \mathbf 1^\top\right)\bigl(\mathbf I-\mathbf m\mathbf m^\top\bigr).
\end{align*}
If $\mathbf z_s=\mathbf m$, the trailing factor gives $\bigl(\mathbf I-\mathbf m\mathbf m^\top\bigr)\mathbf m=\mathbf 0$, hence $\mathbf R_t\mathbf m=\mathbf 0$ and $R_t(\mathbf z_t,\mathbf m)=0$ for all $\mathbf z_t$, i.e., \textsc{[mask]} is absorbing.

If $\mathbf z_s\neq\mathbf m$, then $\mathbf m^\top\mathbf z_s=0$ and $\mathbf 1^\top\mathbf z_s=1$, it follows that $\bigl(\mathbf I-\mathbf m\mathbf m^\top\bigr)\mathbf z_s=\mathbf z_s$, $\bigl(\mathbf I-\mathbf m\mathbf 1^\top\bigr)\mathbf z_s=\mathbf z_s-\mathbf m$, and $\bigl(\mathbf I-\mathbf u\mathbf 1^\top\bigr)\mathbf z_s=\mathbf z_s-\mathbf u$. Hence
\begin{align*}
\mathbf R_t\mathbf z_s
&=\frac{\gamma_t'}{\gamma_t}\bigl(\mathbf z_s-\mathbf m\bigr)
+\frac{\rho_t'}{\rho_t}\bigl(\mathbf z_s-\mathbf u\bigr)
=\Bigl(\frac{\gamma_t'}{\gamma_t}+\frac{\rho_t'}{\rho_t}\Bigr)\mathbf z_s
-\Bigl(\frac{\gamma_t'}{\gamma_t}\mathbf m+\frac{\rho_t'}{\rho_t}\mathbf u\Bigr),
\end{align*}
and therefore
\begin{align*}
R_t(\mathbf z_t,\mathbf z_s)
&=\mathbf z_t^\top\mathbf R_t\mathbf z_s
=\Bigl(\frac{\gamma_t'}{\gamma_t}+\frac{\rho_t'}{\rho_t}\Bigr)\mathbf z_t^\top\mathbf z_s
-\mathbf z_t^\top\Bigl(\frac{\rho_t'}{\rho_t}\mathbf u+\frac{\gamma_t'}{\gamma_t}\mathbf m\Bigr),
\end{align*}
which matches the scalar expression.
\end{proof}

\begin{remark}[degeneration to MDLM]
We recall that the forward rate matrix in MDLM \cite{sahoo2024simple} can be written, under the same column-sum-zero convention, as
\begin{align}
\mathbf R_t^{\mathrm{MDLM}}
=
\frac{(\alpha_t^{\mathrm{MDLM}})'}{\alpha_t^{\mathrm{MDLM}}}
\bigl(\mathbf I-\mathbf m \mathbf 1^\top\bigr)
\label{eq:mdlm_forward_rate_matrix}
\end{align}
When $\rho_t\equiv 1$ (thus $\rho_t'/\rho_t\equiv 0$), the uniform-mixing term in \eqref{eq:Rt-matrix-form} vanishes and the generator reduces to $\mathbf R_t=\frac{\gamma_t'}{\gamma_t}\bigl(\mathbf I-\mathbf m\mathbf 1^\top\bigr)$, i.e., $\mathbf R_t^{\mathrm{MDLM}}$ under the reparameterization $\alpha_t^{\mathrm{MDLM}}=\gamma_t$; see Appendix~\ref{sec:mdlm-ours} for a further discussion of the relation between SCDD and MDLM.
\end{remark}

\section{Discussion of Posterior Distribution and Backward Process}\label{sec:backward}
In this section, we first derive the posterior distribution in \eqref{eq:backward}.  Then, we discuss validity of the backward process defined in \eqref{eq:backwardParam}. Finally, we derive the CTMC backward transition rate.
\subsection{Derivation of Posterior Distribution}
Let $s,t$ be two adjacent time points such that $s<t$.
\paragraph{Case 1: $\mathbf z_{t}\neq\mathbf m$.}

When $\mathbf z_{t} \neq \mathbf m$, we have $\mathbf z_{s}\neq \mathbf m$ since SCDD doesn't allow ``remasking'' during token generation. The posterior $q(\mathbf z_s | \mathbf z_t, \mathbf x)$ is given by Bayes' rule:
\begin{align}
q(\mathbf z_s \mid \mathbf z_t, \mathbf x)&=\frac{q(\mathbf z_s|\mathbf x)}{q(\mathbf z_t | \mathbf x)}q(\mathbf z_t|\mathbf z_s)\nonumber\\
&=
\frac{
\gamma_{s}\left(\rho_s\mathbf 1(\mathbf z_s=\mathbf x)+(1-\rho_s)\frac{1}{K} 
\right) + (1-\gamma_s)1(\mathbf z_s=\mathbf m)}{
\gamma_t\left(\rho_t\mathbf 1(\mathbf z_t=\mathbf x)+(1-\rho_t)\frac{1}{K}\right) + (1-\gamma_t)1(\mathbf z_s=\mathbf m)
}\cdot \frac{\gamma_t}{\gamma_s}\left[\frac{\rho_t}{\rho_s}\mathbf 1(\mathbf z_s=\mathbf z_t)+\frac{\rho_s-\rho_t}{\rho_s}\frac{1}{K}\right] \mathbf 1(\mathbf z_s\neq\mathbf m)\nonumber\\ 
&=(1-\mathbf 1(\mathbf z_s=\mathbf m) )\frac{\rho_s\mathbf 1(\mathbf z_s=\mathbf x)+(1-\rho_s)\frac{1}{K}}{\rho_t\mathbf 1(\mathbf z_t=\mathbf x)+(1-\rho_t)\frac{1}{K}}\left[\frac{\rho_t}{\rho_s}\mathbf 1(\mathbf z_s=\mathbf z_t)+\frac{\rho_s-\rho_t}{\rho_s}\frac{1}{K}\right]\label{eq:backward_v}.
\end{align} 
Note that when $\rho_{s}=0$, we have $\rho_t=0$ since $\rho_t$ is nonnegative and nonincreasing. In this case, we use convention $\frac{\rho_t}{\rho_s}=0$ and the posterior becomes $q(\mathbf z_s|\mathbf z_t,\mathbf x)\equiv\frac{1}{K}$. 

\paragraph{Case 2: $\mathbf z_t = \mathbf m$.} If $\mathbf z_s = \mathbf m$, the posterior is given by
\begin{align}
    q(\mathbf z_s \mid \mathbf z_t, \mathbf x)&=\frac{q(\mathbf z_s |\mathbf x)}{q(\mathbf z_t| \mathbf x )}q(\mathbf z_t |\mathbf z_s)\nonumber \\
    &= \frac{1-\gamma_s}{1-\gamma_t}.
    \label{eq:posCase1}
\end{align}
Similarly, if $\mathbf z_s \neq \mathbf m$, the posterior is given by 
\begin{align}
    q(\mathbf z_s \mid \mathbf z_t, \mathbf x)&=\frac{q(\mathbf z_s|\mathbf x)}{q(\mathbf z_t| \mathbf x )}q(\mathbf z_t|\mathbf z_s )\nonumber \\
    &= \frac{\gamma_s(\rho_s\mathbf 1(\mathbf z_s = \mathbf x) + (1-\rho_s)\frac{1}{K})}{1-\gamma_t}\left(1-\frac{\gamma_t}{\gamma_s}\right).
    \label{eq:posCase2}
\end{align}

Combining \eqref{eq:posCase1} and \eqref{eq:posCase2}, we obtain:
\begin{align}
q(\mathbf z_s \mid \mathbf z_t, \mathbf x)&=
\mathbf 1(\mathbf z_s=\mathbf m)\frac{1-\gamma_s}{1-\gamma_t}
+
\mathbf 1(\mathbf z_s\neq\mathbf m) \frac{\gamma_s(\rho_s\mathbf 1(\mathbf z_s = \mathbf x) + (1-\rho_s)\frac{1}{K})}{1-\gamma_t}\left(1-\frac{\gamma_t}{\gamma_s}\right)\nonumber \\
 &=\mathbf 1(\mathbf z_s=\mathbf m)\frac{1-\gamma_s}{1-\gamma_t}
 + \mathbf 1(\mathbf z_s\neq\mathbf m) \frac{\gamma_s-\gamma_t}{1-\gamma_t}\left(\rho_s\mathbf 1(\mathbf z_s=\mathbf x)+(1-\rho_s)\tfrac{1}{K}\right)\label{eq:backward_m}.
\end{align}
Since $\mathbf z_s$, $\mathbf z_t$, $\mathbf m$ and $\mathbf x$ are all one-hot vectors, we replace the indicator functions with inner products to get \eqref{eq:backward}. This step ensures valid gradient propagation when we replace $\mathbf x$ with $\mathbf x_\theta$, which is not a one-hot vector, in the backward process.
\subsection{Validity of Backward Process}\label{sec:model}
In this section, we show that the model parameterization in \eqref{eq:backwardParam} defines a valid density $p_\theta(\cdot\mid\mathbf z_t)$ for any $\mathbf z_t$.
\begin{lemma}\label{lem:p_theta_is_probability}
The model parameterization in \eqref{eq:backwardParam} defines a valid density $p_\theta(\cdot\mid\mathbf z_t)$ for any $\mathbf z_t$. 
\end{lemma}
\begin{proof}
We write $\mathbf x_\theta=\sum_{i=1}^K c_ie_i$ where $e_i$ is the $i$-th $K+1$-dimensional canonical basis vector, and $c_i=\mathbf x_\theta^\top e_i$. We have $c = (c_1,...,c_K,0)\in \Delta_{K+1}$. 

\paragraph{Case 1: $\mathbf z_t=\mathbf m$.} In this case, the backward process is given by:
\begin{align}
    p_\theta(\mathbf z_s\mid\mathbf z_t) = (1-\mathbf z_s^\top\mathbf m)\frac{\gamma_s - \gamma_t}{1-\gamma_t}(\rho_s\mathbf z_s^\top\mathbf x_\theta + (1-\rho_s)\tfrac{1}{K}) + \mathbf z_s^\top\mathbf m\frac{1-\gamma_s}{1-\gamma_t}
    \label{eq:backProc1}
\end{align}
Summing over the support of $p_\theta$ in \eqref{eq:backProc1} gives:
\begin{align}
    \sum_{\mathbf z_s}p_\theta(\mathbf z_s\mid\mathbf z_t) &=\sum_{\mathbf z_s}\left[(1-\mathbf z_s^\top\mathbf m)\frac{\gamma_s - \gamma_t}{1-\gamma_t}(\rho_s\mathbf z_s^\top\mathbf x_\theta + (1-\rho_s)\tfrac{1}{K}) + \mathbf z_s^\top\mathbf m\frac{1-\gamma_s}{1-\gamma_t}\right] \nonumber\\
    &\overset{(i)}{=}\frac{1-\gamma_s}{1-\gamma_t} + \frac{\gamma_s - \gamma_t}{1-\gamma_t}\left(\rho_s\sum_{\mathbf z_s}\mathbf z_s\mathbf x_\theta + 1-\rho_s\right) \nonumber \\
    &= \frac{1-\gamma_s}{1-\gamma_t} + \frac{\gamma_s - \gamma_t}{1-\gamma_t} \nonumber \\
    &=1,
\end{align}
where $(i)$ is from the fact that $\mathbf x_\theta$ is a probability vector over the vocabulary, and hence $\sum_{\mathbf z_s}\mathbf z_s^\top\mathbf x_\theta = 1$. Therefore, $p_\theta(\cdot |\mathbf z_t)$ defines a probability density in this case.

\paragraph{Case 2: $\mathbf z_t\neq \mathbf m$.}
In this case, we have $\mathbf z_s\neq \mathbf m$ since \textsc{[mask]} is absorbing. Using the fact $\frac{a_1+a_2}{b_1+b_2}=\frac{b_1}{b_1+b_2}\frac{a_1}{b_1}+\frac{b_2}{b_1+b_2}\frac{a_2}{b_2}$, we obtain: 
\begin{align}
\sum_{\mathbf z_s} p_\theta(\mathbf z_s \mid \mathbf z_t)
&=
\sum_{\mathbf z_s\neq\mathbf m}
\frac{
\rho_s \mathbf z_s^\top\mathbf x_\theta
+
(1-\rho_s)\tfrac{1}{K}
}{
\rho_t \mathbf z_t^\top\mathbf x_\theta
+
(1-\rho_t)\tfrac{1}{K}
}
\Big(
\frac{\rho_t}{\rho_s} \mathbf z_s^\top \mathbf z_t
+
\frac{\rho_s-\rho_t}{\rho_s}\tfrac{1}{K}
\Big) \nonumber\\
&=\sum_{\mathbf z_s\neq\mathbf m}\frac{\sum_{i}c_i (\rho_s \mathbf z_s^\top e_i
+
(1-\rho_s)\tfrac{1}{K})}{\sum_i c_i (\rho_t \mathbf z_t^\top e_i
+
(1-\rho_t)\tfrac{1}{K})}\Big(
\frac{\rho_t}{\rho_s} \mathbf z_s^\top \mathbf z_t
+\frac{\rho_s-\rho_t}{\rho_s}
\tfrac{1}{K}
\Big)\nonumber\\
&=\sum_{\mathbf z_s\neq\mathbf m}\sum_i \underbrace{\left(\frac{c_i (\rho_t \mathbf z_t^\top e_i
+
(1-\rho_t)\tfrac{1}{K})}{\sum_i c_i (\rho_t \mathbf z_t^\top e_i
+
(1-\rho_t)\tfrac{1}{K})}\right)}_{\kappa_i} \frac{\rho_s\mathbf z_s^\top e_i
+
(1-\rho_s)\tfrac{1}{K}}{\rho_t \mathbf z_t^\top e_i
+
(1-\rho_t)\tfrac{1}{K}}\Big(
\frac{\rho_t}{\rho_s} \mathbf z_s^\top \mathbf z_t
+\frac{\rho_s-\rho_t}{\rho_s}
\tfrac{1}{K}
\Big)\nonumber\\
&=\sum_i \kappa_i\sum_{\mathbf z_s\neq\mathbf m}\frac{\rho_s\mathbf z_s^\top e_i
+
(1-\rho_s)\tfrac{1}{K}}{\rho_t\mathbf z_t^\top e_i
+
(1-\rho_t)\tfrac{1}{K}}\Big(\frac{\rho_t}{\rho_s} \mathbf z_s^\top \mathbf z_t
+\frac{\rho_s-\rho_t}{\rho_s}
\tfrac{1}{K}
\Big)\nonumber\\
&=\sum_i \kappa_i=1\nonumber 
\end{align}
Thus, $p_\theta(\cdot |\mathbf z_t)$ defines a probability density in this case. 
\end{proof}

\subsection{Derivation of Continuous-time Backward Transition Rate}
\label{sec:backwardrate}
In this section, we derive the transition rate for the backward process defined in \eqref{eq:backwardParam} similarly to the forward transition rate \eqref{eq:forward_rate}. For each $t\in(0,1]$ and $\mathbf z_t,\mathbf z_s\in\mathcal{V}$, define
\begin{align}
\tilde{Q}^\theta_t(\mathbf z_s,\mathbf z_t)
=
\begin{cases}
-\dfrac{1}{K}\dfrac{\rho'_t}{\rho_t}
\dfrac{\rho_t\,\mathbf z_s^\top \mathbf x_\theta(\mathbf z_t,t)+(1-\rho_t)\tfrac{1}{K}}
{\rho_t\,\mathbf z_t^\top \mathbf x_\theta(\mathbf z_t,t)+(1-\rho_t)\tfrac{1}{K}}
& \text{if } \mathbf z_s\neq \mathbf z_t,\ \mathbf z_t\neq \mathbf m\\[10pt]
\dfrac{1}{K}\dfrac{\rho'_t}{\rho_t}
\dfrac{1-(\rho_t\,\mathbf z_t^\top \mathbf x_\theta(\mathbf z_t,t)+(1-\rho_t)\tfrac{1}{K})}
{\rho_t\,\mathbf z_t^\top \mathbf x_\theta(\mathbf z_t,t)+(1-\rho_t)\tfrac{1}{K}}
& \text{if } \mathbf z_s=\mathbf z_t,\ \mathbf z_t\neq \mathbf m\\[10pt]
-\dfrac{\gamma'_t}{1-\gamma_t}
\left(\rho_t\,\mathbf z_s^\top \mathbf x_\theta(\mathbf z_t,t)+(1-\rho_t)\tfrac{1}{K}\right)
& \text{if } \mathbf z_s\neq \mathbf z_t,\ \mathbf z_t=\mathbf m\\[10pt]
\dfrac{\gamma'_t}{1-\gamma_t}
& \text{if } \mathbf z_t=\mathbf z_s=\mathbf m \\
0& \text{if } \mathbf z_t\neq\mathbf m, \ \mathbf z_s=\mathbf m
\end{cases}
\label{eq:backward_rate}
\end{align}

\begin{proposition}
Let $\tilde{Q}^\theta_t$ be defined as in~\eqref{eq:backward_rate}. Then the following statements hold:
\begin{itemize}
    \item $\tilde{Q}^\theta_t$ is a valid infinitesimal generator in the sense of~\citet{gat2024discrete}; in particular,
    $$
    \sum_{\mathbf z_s}\tilde{Q}^\theta_t(\mathbf z_s,\mathbf z_t)=0,
    \qquad
    \tilde{Q}^\theta_t(\mathbf z_s,\mathbf z_t)\ge 0, \forall\, \mathbf z_s\neq \mathbf z_t.
    $$
    \item  The discrete-time backward process defined in \eqref{eq:backwardParam} coincide with $\tilde{Q}^\theta_t$. In particular, for adjacent time points $s<t$ with $\Delta t=t-s$, we have
    $$
    p_\theta(\mathbf z_s\mid \mathbf z_t)
    =
    \delta_{\mathbf z_t,\mathbf z_s}
    +
    \Delta t\,\tilde{Q}^\theta_t(\mathbf z_s,\mathbf z_t)
    +
    o(\Delta t).
    $$
\end{itemize}
\end{proposition}

\begin{proof}
The proof follows the same argument as in Lemma~\ref{lem:forward-rate}. First, we show that $\tilde{Q}^\theta_t$ is a valid generator.

Fix an arbitrary $\mathbf z_t$. Since $\rho_t',\,\gamma_t' \le 0$, all off-diagonal entries of $\tilde{Q}^\theta_t$ are non-negative, i.e.,
\[
\tilde{Q}^\theta_t(\mathbf z_s,\mathbf z_t)\ge 0,\qquad \forall\, \mathbf z_s\neq \mathbf z_t.
\]
Moreover, by construction of $\tilde{Q}^\theta_t$ we have the row-sum condition
\[
\sum_{\mathbf z_s}\tilde{Q}_t^\theta(\mathbf z_s,\mathbf z_t)=0.
\]
Therefore, $\tilde{Q}^\theta_t$ is a valid infinitesimal generator in the sense of~\citet{gat2024discrete}.


Next, we show that  $\tilde{Q}^\theta_t$ is indeed the backward transition kernel of the backward process \eqref{eq:backwardParam}.
Let $s=t-\Delta t$ with $\Delta t>0$. Then by first-order Taylor expansion we have
\begin{align}
\rho_s = \rho_t - \rho_t'\Delta t + o(\Delta t),\qquad
\gamma_s  = \gamma_t - \gamma_t'\Delta t + o(\Delta t).
\label{eq:taylor}
\end{align}

We now verify that the discrete backward kernel~\eqref{eq:backward} satisfies
\begin{align}
p_\theta(\mathbf z_s\mid \mathbf z_t)
=
\delta_{\mathbf z_s,\mathbf z_t}
+
\Delta t\,\tilde{Q}^\theta_t(\mathbf z_s,\mathbf z_t)
+
o(\Delta t)\nonumber
\end{align}
case by case.

\paragraph{Case 1. $\mathbf z_t,\mathbf z_s\neq \mathbf m,\ \mathbf z_s\neq \mathbf z_t$.}
In this case, $\mathbf z_s^\top \mathbf m=0$ and $\mathbf z_s^\top \mathbf z_t=0$. 
Thus, we have 
\begin{align}
p_\theta(\mathbf z_s\mid \mathbf z_t)
&=\frac{\rho_s\, \mathbf z_s^\top \mathbf x_\theta+(1-\rho_s)\tfrac{1}{K}}
{\rho_t\, \mathbf z_t^\top \mathbf x_\theta+(1-\rho_t)\tfrac{1}{K}}
\left(\frac{\rho_s-\rho_t}{\rho_s}\tfrac{1}{K}\right)\nonumber\\
&=\left(\frac{\rho_t\, \mathbf z_s^\top \mathbf x_\theta+(1-\rho_t)\tfrac{1}{K}}{\rho_t\, \mathbf z_t^\top \mathbf x_\theta+(1-\rho_t)\tfrac{1}{K}}
-\frac{\rho_t'(\tfrac{1}{K}-\mathbf z_s^\top \mathbf x_\theta)}{\rho_t\, \mathbf z_t^\top \mathbf x_\theta+(1-\rho_t)\tfrac{1}{K}}\Delta t+o(\Delta t)\right)
\left(\frac{\rho_t'}{\rho_t}(-\tfrac{1}{K})\Delta t+o(\Delta t)\right)\nonumber\\
&=\Delta t\left(
-\frac{\rho_t'}{\rho_t}\tfrac{1}{K}
\frac{\rho_t\, \mathbf z_s^\top \mathbf x_\theta+(1-\rho_t)\tfrac{1}{K}}
{\rho_t\, \mathbf z_t^\top \mathbf x_\theta+(1-\rho_t)\tfrac{1}{K}}
\right)
+o(\Delta t).
\label{pf:ctmc-backward-case1}
\end{align}

\paragraph{Case 2. $\mathbf z_t=\mathbf z_s\neq \mathbf m$.}
We have 
\begin{align}
p_\theta(\mathbf z_s\mid \mathbf z_t)
&=\frac{\rho_s\, \mathbf z_t^\top \mathbf x_\theta+(1-\rho_s)\tfrac{1}{K}}
{\rho_t\, \mathbf z_t^\top \mathbf x_\theta+(1-\rho_t)\tfrac{1}{K}}
\left(\frac{\rho_t}{\rho_s}+\frac{\rho_s-\rho_t}{\rho_s}\tfrac{1}{K}\right)\nonumber\\
&=\left(1+\frac{\rho_t'(\tfrac{1}{K}-\mathbf z_t^\top \mathbf x_\theta)}{\rho_t\, \mathbf z_t^\top \mathbf x_\theta+(1-\rho_t)\tfrac{1}{K}}\Delta t+o(\Delta t)\right)
\left(1+\frac{\rho_t'}{\rho_t}(1-\tfrac{1}{K})\Delta t+o(\Delta t)\right)\nonumber\\
&=1+\Delta t\left(
\frac{\rho_t'(\tfrac{1}{K}-\mathbf z_t^\top \mathbf x_\theta)}{\rho_t\, \mathbf z_t^\top \mathbf x_\theta+(1-\rho_t)\tfrac{1}{K}}
+\frac{\rho_t'}{\rho_t}(1-\tfrac{1}{K})
\right)
+o(\Delta t)\nonumber\\
&=1+\Delta t\frac{\rho_t'}{\rho_t}\tfrac{1}{K}
\frac{1-(\rho_t\, \mathbf z_t^\top \mathbf x_\theta+(1-\rho_t)\tfrac{1}{K})}
{\rho_t\, \mathbf z_t^\top \mathbf x_\theta+(1-\rho_t)\tfrac{1}{K}}
+o(\Delta t).
\label{pf:ctmc-backward-case2}
\end{align}

\paragraph{Case 3. $\mathbf z_s\neq \mathbf m,\ \mathbf z_t=\mathbf m$.}
We have 
\begin{align}
p_\theta(\mathbf z_s\mid \mathbf z_t)
&=\frac{\gamma_s-\gamma_t}{1-\gamma_t}
\bigl(\rho_s\, \mathbf z_s^\top \mathbf x_\theta+(1-\rho_s)\tfrac{1}{K}\bigr)\nonumber\\
&=-\Delta t\frac{\gamma_t'}{1-\gamma_t}
\bigl(\rho_t\, \mathbf z_s^\top \mathbf x_\theta+(1-\rho_t)\tfrac{1}{K}
+\rho_t'(\tfrac{1}{K}-\mathbf z_s^\top \mathbf x_\theta)\Delta t\bigr)\nonumber\\
&=-\Delta t\frac{\gamma_t'}{1-\gamma_t}
\bigl(\rho_t\, \mathbf z_s^\top \mathbf x_\theta+(1-\rho_t)\tfrac{1}{K}\bigr)
+o(\Delta t).
\label{pf:ctmc-backward-case3}
\end{align}

\paragraph{Case 4. $\mathbf z_s=\mathbf z_t=\mathbf m$.}
We have 
\begin{align}
p_\theta(\mathbf z_s\mid \mathbf z_t)
&=\frac{1-\gamma_s}{1-\gamma_t}
=1+\Delta t\frac{\gamma_t'}{1-\gamma_t}+o(\Delta t).
\label{pf:ctmc-backward-case4}
\end{align}
\paragraph{Case 5: $\mathbf z_t \neq \mathbf m$, $\mathbf z_s = \mathbf m$.} The last case holds from the fact that SCDD doesn't allow remasking.

In all cases, $p_\theta(\mathbf z_s\mid \mathbf z_t)$ aligns with the process implied by $\tilde Q^\theta_t(\mathbf z_t,\mathbf z_s)$ in~\eqref{eq:backward_rate}, which completes the proof.
\end{proof}

\section{Analysis of Training Loss}\label{sec:loss}
\subsection{Discrete-time NELBO}
Following standard variational arguments on diffusion models \cite{sohl2015deep, sahoo2024simple}, we start from the classical ELBO:
\begin{align}
\mathbb E[-\log p_\theta(\mathbf x)]
&\leq
\mathbb{E}_{q}
\left[
-\log p_\theta(\mathbf x, \mathbf z_{0:1})
+
\log q(\mathbf z_{0:1} \mid \mathbf x)
\right] \nonumber\\
&=
\underbrace{
-\mathbb{E}_{q(\mathbf x,\mathbf z_{0:1})} \big[ \log p_\theta(\mathbf x \mid \mathbf z_0) \big]
}_{\mathcal L^T_{\text{reconstruction}}} +
\underbrace{
\mathbb E_{q(\mathbf x,\mathbf z_{0:1})}\left[D_{\mathrm{KL}}
\big(
q(\mathbf z_1 \mid \mathbf x)
\|
p_\theta(\mathbf z_1)
\big)\right]
}_{\mathcal L_{\text{prior}}} \nonumber\\&+
\underbrace{
\mathbb E_{q(\mathbf x,\mathbf z_{0:1})}\left[\sum_{i=1}^T
D_{\mathrm{KL}}
\big(
q(\mathbf z_{t_{i-1}} \mid \mathbf z_{t_i}, \mathbf x)
\|
p_\theta(\mathbf z_{t_{i-1}} \mid \mathbf z_{t_i})
\big)\right]
}_{\mathcal L^T_{\text{diffusion}}}=:\mathcal L^T_{\text{NELBO}}.\label{eq:elbo}
\end{align}
In this section, we explicitly calculate the the diffusion loss $\mathcal L^T_{\text{diffusion}}$. As before, we use $t,s$ to denote two adjacent time points with $t>s$. We first rewrite the summation over $i$ as an expectation over all time points: 
\begin{align}
    \mathcal L^T_{\text{diffusion}}&=\mathbb E_{q}\left[\sum_{i=1}^T
D_{\mathrm{KL}}
\big(
q(\mathbf z_{t_{i-1}} \mid \mathbf z_{t_i}, \mathbf x)
\|
p_\theta(\mathbf z_{t_{i-1}} \mid \mathbf z_{t_i})
\big)\right] \nonumber\\
&=\frac{1}{T}\sum_{i=1}^T\mathbb E_{q}\left[
TD_{\mathrm{KL}}
\big(
q(\mathbf z_{t_{i-1}} \mid \mathbf z_{t_i}, \mathbf x)
\|
p_\theta(\mathbf z_{t_{i-1}} \mid \mathbf z_{t_i})
\big)\right] \nonumber\\
&=\mathbb E_{t\sim\mathcal U\{t_1,...,t_T\}}\mathbb E_{q}\left[
TD_{\mathrm{KL}}
\big(
q(\mathbf z_s \mid \mathbf z_t, \mathbf x)
\|
p_\theta(\mathbf z_s \mid \mathbf z_t)
\big)\right]
\end{align}

Since
$$
\mathcal D_{\mathrm{KL}}(q\|p_\theta)
=
\mathbb E_q[\log q] - \mathbb E_q[\log p_\theta],
$$
the first (negative entropy) term in KL-divergence is independent of $\theta$, and minimizing $\mathcal{L}^T_{\text{diffusion}}$ is equivalent to minimizing the expectation of second cross entropy term. With a little abuse of notation, we drop the first term and redefine the $\mathcal L^T_{\text{diffusion}}$ as

\begin{align}
\mathcal L^T_{\text{diffusion}}
&=
-\mathbb E_{t\sim\mathcal U\{t_1,...,t_T\}}\mathbb E_{q}\left[
T
\log p_\theta(\mathbf z_s \mid \mathbf z_t)
\right]
\end{align}
When $\mathbf z_s=\mathbf m$, we have $\mathbf z_t=\mathbf m$ and
$p_\theta(\mathbf z_s\mid \mathbf z_t)=\frac{1-\gamma_s}{1-\gamma_t}$.
This case is independent of $\theta$ and will be ignored in optimization. We therefore restrict to the case $\mathbf z_s\neq\mathbf m$.

\paragraph{Case 1: $\mathbf z_t\neq\mathbf m$.}
Substituting $p_\theta(\mathbf z_s\mid \mathbf z_t)$ yields
\begin{align}
\log p_\theta(\mathbf z_s\mid \mathbf z_t)
&=
\log\left(
\frac{\rho_s\mathbf z_s^\top \mathbf x_\theta + (1-\rho_s)\tfrac{1}{K}}
{\rho_t\mathbf z_t^\top \mathbf x_\theta + (1-\rho_t)\tfrac{1}{K}}
\right)
+
\log\left(\frac{\rho_t}{\rho_s} \mathbf z_s^\top \mathbf z_t +\frac{\rho_s-\rho_t}{\rho_s}\tfrac{1}{K}\right)
\nonumber
\end{align}

The second term in $\log p_\theta(\mathbf z_s\mid\mathbf z_t)$ is independent of $\theta$ and drops from the loss.
Thus, up to an additive constant,
\begin{align}
\mathcal L^T_{\text{diffusion}}\mid_{\mathbf z_t\neq \mathbf m}\nonumber
&=-\mathbb E_{t\sim\mathcal U\{t_1,...,t_T\}}\mathbb E_{q} 
\left[
T\sum_{\mathbf z_s\neq\mathbf m}
\frac{\rho_s\mathbf z_s^\top \mathbf x
+(1-\rho_s)\tfrac{1}{K}}{\rho_t \mathbf z_t^\top \mathbf x
+(1-\rho_t)\tfrac{1}{K}}
\Big(\frac{\rho_t}{\rho_s}\mathbf z_s^\top \mathbf z_t+\frac{\rho_s-\rho_t}{\rho_s}\tfrac{1}{K}\Big)
\log\frac{\rho_s\mathbf z_s^\top \mathbf x_\theta + (1-\rho_s)\tfrac{1}{K}}{\rho_t \mathbf z_t^\top \mathbf x_\theta +(1-\rho_t)\tfrac{1}{K}} 
\right]\\
&=-\mathbb E_{t\sim\mathcal U\{t_1,\dots,t_T\}}\mathbb E_{q}\left[ T\sum_{\mathbf v\neq \mathbf m}
\frac{\rho_s\mathbf v^\top \mathbf x
+(1-\rho_s)\frac{1}{K}}{\rho_t \mathbf z_{t}^\top \mathbf x
+(1-\rho_t)\frac{1}{K}}
\Big(\frac{\rho_t}{\rho_s}\mathbf v^\top \mathbf z_{t}+\frac{\rho_s-\rho_t}{\rho_s}\tfrac{1}{K}\Big)
\log\frac{\rho_s\mathbf v^\top \mathbf x_\theta + (1-\rho_s)\frac{1}{K}}
{\rho_t \mathbf z_{t}^\top \mathbf x_\theta +(1-\rho_t)\frac{1}{K}}\right],
\label{eq:loss_v} 
\end{align}
up to some $\theta$-independent additive constant. We replaced $\mathbf z_s$ with $\mathbf v$ in the last line to make it clear that the summation doesn't depend on the value of $\mathbf z_s$.
\paragraph{Case 2: $\mathbf z_t=\mathbf m$.}
In this case,
$$
\log p_\theta(\mathbf z_s\mid \mathbf z_t = \mathbf m)
=\log\Big(\rho_s \mathbf z_s^\top \mathbf x_\theta+(1-\rho_s)\tfrac{1}{K}\Big)
+\log\Big(\tfrac{\gamma_s-\gamma_t}{1-\gamma_t}\Big),
$$

Dropping the second term in $\log p_\theta(\mathbf z_s\mid\mathbf z_t)$ and substituting into the cross-entropy objective yields
\begin{align}
\mathcal L^T_{\text{diffusion}}\mid_{\mathbf z_t}
&=-\mathbb E_{t\sim\mathcal U\{t_1,\dots,t_T\}}\mathbb E_{q}\left[T\sum_{\mathbf z_s\neq\mathbf m}\frac{\gamma_s(\rho_s\mathbf z_s^{\top}\mathbf x+(1-\rho_s)\frac{1}{K})}{1-\gamma_t}\frac{\gamma_s-\gamma_t}{\gamma_s}\log\Big(\rho_s \mathbf z_s^\top \mathbf x_\theta+(1-\rho_s)\tfrac{1}{K}\Big)\right]\nonumber\\ 
&=-\mathbb E_{t\sim\mathcal U\{t_1,\dots,t_T\}}\mathbb E_{q}\left[T\sum_{\mathbf v\neq\mathbf m}\frac{\gamma_s-\gamma_t}{1-\gamma_t}(\rho_s\mathbf v^{\top}\mathbf x+(1-\rho_s)\tfrac{1}{K})\log (\rho_s\mathbf v^\top \mathbf x_\theta+(1-\rho_s)\tfrac{1}{K})\right],\label{eq:loss_m}
\end{align}
up to some $\theta$-independent additive constant. 
\subsection{Continuous-time NELBO}
Next, we show that reconstruction loss and prior loss vanishes as $T\to\infty$, and explicitly derive $\mathcal L^\infty_{\text{diffusion}}$.
\paragraph{Reconstruction Loss} As $T\to\infty$, $\rho_0,\gamma_0\to1^-$ by the design of noising schedule. Therefore, we have 
\begin{align}
\mathbf z_0 &\sim \text{Cat}\Bigl(\cdot;\gamma_0 \bigl(\rho_0 \mathbf x+(1-\rho_0)\mathbf u\bigr)+(1-\gamma_0)\mathbf{m}\Bigr) \nonumber \\
&\overset{d}{\Longrightarrow}\text{Cat}\Bigl(\cdot;\mathbf x\Bigr),
\end{align}
where ``$\overset{d}{\Longrightarrow}$'' stands for ``converge in distribution''. Therefore, we have the reconstruction loss vanishes as follows:
\begin{align}
    \lim_{T\to\infty}\mathcal L^T_{\text{reconstruction}} &= \lim_{T\to\infty}\mathbb E_q[-\log p_\theta(\mathbf x\mid\mathbf z_0)]\nonumber\\
    &\overset{(i)}{=}\lim_{T\to\infty}\mathbb E_q\left[-\log\frac{\mathbf \rho_{-1} \mathbf x^\top \mathbf x_\theta(\mathbf z_0,0)+(1-\rho_{-1})\frac{1}{K}}{\rho_0 \mathbf z_0^\top \mathbf x_\theta(\mathbf z_0,0)+(1-\rho_0)\frac{1}{K}}\left(\frac{\rho_0}{\rho_{-1}}\mathbf x^\top \mathbf z_0+\frac{\rho_{-1}-\rho_0}{\rho_{-1}}\frac{1}{K}\right)\right]\nonumber \\
    &=\mathbb E_q\left[-\log\frac{\mathbf x^\top \mathbf x_\theta(\mathbf x,0)}{\mathbf x^\top \mathbf x_\theta(\mathbf x,0)}\left(\mathbf x^\top \mathbf x\right)\right]\nonumber \\
    &=\mathbb E_q[-\log 1] = 0,
\end{align}
where $(i)$ is from plugging $\mathbf z_t = \mathbf z_0, t=0$ into \eqref{eq:backwardParam}.
\paragraph{Prior Loss}
At $t=1$, the forward marginal distribution becomes $q(\mathbf z_1 \mid \mathbf x) = \mathbf 1(\mathbf z_1 = \mathbf m)$, indicating that the forward noising process ultimately transforms clean data $\mathbf x$ to $\mathbf m$ at $t=1$. We also choose the prior to be $p_\theta(\mathbf z_1) = q(\mathbf z_1 \mid \mathbf x) = \mathbf 1(\mathbf z_1 = \mathbf m)$ so that the backward process starts from $\mathbf m$. Therefore, we have $\mathcal L^T_{\text{prior}}\equiv0$. 
\paragraph{Diffusion Loss}
In this section, we discuss the continuous-time version of the loss function in \eqref{eq:loss_zt}. We assume that $\rho_t$ and $\gamma_t$ are decreasing, differentiable functions of time.

\begin{proposition}
Pick adjacent time points $s<t$ and let $\Delta t=t-s$ be the time step of the above discrete stochastic process.
Then as $\Delta t \to 0$ or, equivalently,  $T\to\infty$, we have
\begin{align}
\mathcal{L}^\infty_{\text{diffusion}}
=
\begin{cases}
\mathbb E_{t\sim\mathcal U[0,1]}\mathbb E_q\left[
\sum_{\mathbf v\neq \mathbf z_t,\mathbf m }\left(
\frac{(\rho_t \mathbf v^\top \mathbf x+(1-\rho_t)\frac{1}{K})\frac{\rho'_t}{\rho_t}\frac{1}{K}}
{\rho_t \mathbf z_t^\top \mathbf x+(1-\rho_t)\frac{1}{K}}
\log
\frac{\rho_t \mathbf v^\top \mathbf x_\theta+(1-\rho_t)\frac{1}{K}}
{\rho_t \mathbf z_t^\top \mathbf x_\theta+(1-\rho_t)\frac{1}{K}}\right)
-
\frac{\rho'_t(-\mathbf z_t^\top \mathbf x_\theta+\frac{1}{K})}
{\rho_t \mathbf z_t^\top \mathbf x_\theta+(1-\rho_t)\frac{1}{K}}\right],
& \text{if } \mathbf z_t\neq \mathbf m,
\\[10pt]
\mathbb E_{t\sim\mathcal U[0,1]} \mathbb E_q\left[\sum_{\mathbf v\neq\mathbf m}
\frac{\gamma_t'}{1-\gamma_t}
(\rho_t \mathbf v^\top \mathbf x+(1-\rho_t)\frac{1}{K})
\log(\rho_t \mathbf v^\top \mathbf x_\theta+(1-\rho_t)\frac{1}{K})\right],
& \text{if } \mathbf z_t=\mathbf m.
\end{cases}\label{eq:loss_continuous}
\end{align}
\end{proposition}

\begin{proof}
We distinguish two cases of $\mathbf z_t$ to calculate the continuous-time loss function.
\paragraph{Case 1: $\mathbf z_t\neq \mathbf m$.} If $\mathbf z_t\neq \mathbf m$, using the first-order Taylor expansion of $\rho_s$ and $\gamma_s$, we have
\begin{align}
\rho_s \mathbf z_s^\top\mathbf x+(1-\rho_s)\tfrac{1}{K}
&=
(\rho_t-\rho'_t\Delta t)\mathbf z_s^\top\mathbf x
+(1-\rho_t+\rho'_t\Delta t)\tfrac{1}{K}
+o(\Delta t)
\nonumber\\
&=
(\rho_t \mathbf z_s^\top\mathbf x+(1-\rho_t)\tfrac{1}{K})
+\Delta t\rho'_t(-\mathbf z_s^\top\mathbf x+\tfrac{1}{K})
+o(\Delta t),
\nonumber\\
\log
\frac{\rho_s \mathbf z_s^\top\mathbf x_\theta+(1-\rho_s)\tfrac{1}{K}}
{\rho_t \mathbf z_t^\top\mathbf x_\theta+(1-\rho_t)\tfrac{1}{K}}
&=
\log
\frac{\rho_t \mathbf z_s^\top\mathbf x_\theta+(1-\rho_t)\frac{1}{K}}
{\rho_t \mathbf z_t^\top\mathbf x_\theta+(1-\rho_t)\frac{1}{K}}
+
\Delta t
\frac{\rho'_t(-\mathbf z_s^\top\mathbf x_\theta+\frac{1}{K})}
{\rho_t \mathbf z_s^\top\mathbf x_\theta+(1-\rho_t)\frac{1}{K}}
+o(\Delta t),
\nonumber\\
\frac{\rho_t}{\rho_s}\mathbf z_s^\top \mathbf z_t
+
\frac{\rho_s-\rho_t}{\rho_s}\tfrac{1}{K}
&=
\mathbf z_s^\top \mathbf z_t
+
\frac{\rho'_t}{\rho_t}\Delta t
(\mathbf z_s^\top \mathbf z_t-\tfrac{1}{K})
+o(\Delta t).
\end{align}

Substituting the above expansions into the discrete loss \eqref{eq:loss_v}, 
\begin{align}
\mathcal L^T_{\text{diffusion}}\mid_{\mathbf z_t\neq \mathbf m} = -&T\mathbb E_{t\sim\mathcal U\{t_1,\dots,t_T\}}\mathbb E_q \sum_{\mathbf z_s\neq\mathbf m}\Bigg[ 
\mathbf z_s^\top \mathbf z_t
\frac{\rho_t \mathbf z_s^\top\mathbf x+(1-\rho_t)\tfrac{1}{K}}
{\rho_t \mathbf z_t^\top\mathbf x+(1-\rho_t)\tfrac{1}{K}}
\Bigg(
\log
\frac{\rho_t \mathbf z_s^\top\mathbf x_\theta+(1-\rho_t)\frac{1}{K}}
{\rho_t \mathbf z_t^\top\mathbf x_\theta+(1-\rho_t)\frac{1}{K}}
+
\Delta t
\frac{\rho'_t(-\mathbf z_s^\top\mathbf x_\theta+\frac{1}{K})}
{\rho_t \mathbf z_s^\top\mathbf x_\theta+(1-\rho_t)\frac{1}{K}}
\Bigg)
\nonumber\\
&+
\Delta t
\frac{
\rho'_t(-\mathbf z_s^\top\mathbf x+\frac{1}{K})\mathbf z_s^\top \mathbf z_t
+
(\rho_t \mathbf z_s^\top\mathbf x+(1-\rho_t)\frac{1}{K})
\frac{\rho'_t}{\rho_t}(\mathbf z_s^\top \mathbf z_t-\frac{1}{K})
}
{\rho_t \mathbf z_t^\top\mathbf x+(1-\rho_t)\frac{1}{K}}
\log
\frac{\rho_t \mathbf z_s^\top\mathbf x_\theta+(1-\rho_t)\frac{1}{K}}
{\rho_t \mathbf z_t^\top\mathbf x_\theta+(1-\rho_t)\frac{1}{K}}
+o(\Delta t)\Bigg].\nonumber 
\end{align}

We distinguish two cases when doing summation over $\mathbf z_s$ to simplify the above expansion: (i) if $\mathbf z_s=\mathbf z_t$, all logarithmic terms vanish and the loss reduces to
\begin{align}
-T\Delta t\mathbb E_{t\sim\mathcal U\{t_1,\dots,t_T\}}\mathbb E_q \left[
\frac{\rho'_t(-\mathbf z_t^\top\mathbf x_\theta+\frac{1}{K})}
{\rho_t \mathbf z_t^\top\mathbf x_\theta+(1-\rho_t)\frac{1}{K}}\right] + o(1) = -\mathbb E_{t\sim\mathcal U\{t_1,\dots,t_T\}}\mathbb E_q\left[
\frac{\rho'_t(-\mathbf z_t^\top\mathbf x_\theta+\frac{1}{K})}
{\rho_t \mathbf z_t^\top\mathbf x_\theta+(1-\rho_t)\frac{1}{K}}\right] + o(1);\nonumber 
\end{align}
(ii) if $\mathbf z_s\neq \mathbf z_t$, the first term inside of expectation and summation vanishes and we obtain
\begin{align}
\mathbb E_{t\sim\mathcal U\{t_1,\dots,t_T\}}\mathbb E_q\sum_{\mathbf z_s \neq \mathbf z_t, \mathbf m}\left[
\frac{(\rho_t \mathbf z_s^\top\mathbf x+(1-\rho_t)\frac{1}{K})
\frac{\rho'_t}{\rho_t}\frac{1}{K}}
{\rho_t \mathbf z_t^\top\mathbf x+(1-\rho_t)\frac{1}{K}}
\log
\frac{\rho_t \mathbf z_s^\top\mathbf x_\theta+(1-\rho_t)\frac{1}{K}}
{\rho_t \mathbf z_t^\top\mathbf x_\theta+(1-\rho_t)\frac{1}{K}}\right] + o(1).
\end{align}

Combining the two cases, letting $T\to\infty$, and replacing the $\mathbf z_s$ notation with $\mathbf v$ yields the desired limit for $\mathbf z_t\neq\mathbf m$:

$$\mathcal L^\infty_{\text{diffusion}}\mid_{\mathbf z_t \neq \mathbf m} = \mathbb E_{t\sim\mathcal U[0,1]}\mathbb E_q\left[
\sum_{\mathbf v\neq \mathbf z_t,\mathbf m }\left(
\frac{(\rho_t \mathbf v^\top\mathbf x+(1-\rho_t)\frac{1}{K})\frac{\rho'_t}{\rho_t}\frac{1}{K}}
{\rho_t \mathbf z_t^\top\mathbf x+(1-\rho_t)\frac{1}{K}}
\log
\frac{\rho_t \mathbf v^\top\mathbf x_\theta+(1-\rho_t)\frac{1}{K}}
{\rho_t \mathbf z_t^\top\mathbf x_\theta+(1-\rho_t)\frac{1}{K}}\right)
-
\frac{\rho'_t(-\mathbf z_t^\top\mathbf x_\theta+\frac{1}{K})}
{\rho_t \mathbf z_t^\top\mathbf x_\theta+(1-\rho_t)\frac{1}{K}}\right].$$

\paragraph{Case 2: $\mathbf z_t=\mathbf m$.} Using the first-order Taylor expansions to get
\begin{align}
\frac{\gamma_t-\gamma_s}{1-\gamma_t}
&=
\Delta t\frac{\gamma_t'}{1-\gamma_t},
\nonumber\\
\rho_s \mathbf z_s^\top\mathbf x+(1-\rho_s)\tfrac{1}{K}
&=
(\rho_t \mathbf z_s^\top\mathbf x+(1-\rho_t)\tfrac{1}{K})
+\Delta t\rho'_t(-\mathbf z_s^\top\mathbf x+\tfrac{1}{K})
+o(\Delta t),
\nonumber\\
\log(\rho_s \mathbf z_s^\top\mathbf x_\theta+(1-\rho_s)\tfrac{1}{K})
&=
\log(\rho_t \mathbf z_s^\top\mathbf x_\theta+(1-\rho_t)\tfrac{1}{K})
+
\Delta t
\frac{\rho'_t(-\mathbf z_s^\top\mathbf x_\theta+\frac{1}{K})}
{\rho_t \mathbf z_s^\top\mathbf x_\theta+(1-\rho_t)\frac{1}{K}}
+o(\Delta t).
\end{align}

Substituting into the discrete loss \eqref{eq:loss_m} yields
\begin{align}
\mathcal L^\infty_{\text{diffusion}}\mid_{\mathbf z_t = \mathbf m} = \mathbb E_{t\sim\mathcal U[0,1]}\mathbb E_q\left[
\sum_{\mathbf z_s\neq\mathbf m}
\frac{\gamma_t'}{1-\gamma_t}
(\rho_t \mathbf z_s^\top\mathbf x+(1-\rho_t)\tfrac{1}{K})
\log(\rho_t \mathbf z_s^\top\mathbf x_\theta+(1-\rho_t)\tfrac{1}{K})\right],
\end{align}
which completes the proof.

\end{proof}
\section{Connections Between SCDD and Other Models}

\subsection{SCDD and MDLM} \label{sec:mdlm-ours}
In this section, we discuss the relation between SCDD and MDLM \cite{sahoo2024simple}, and show SCDD can be reduced to MDLM under specific parameter settings. Firstly, recall several facts of MDLM (see \citet{sahoo2024simple} for full derivations):

\paragraph{Marginal Distribution.} The marginal distribution at time $t\in[0,1]$ is
\begin{align}
q(\mathbf z_t|\mathbf x)=\text{Cat}(\mathbf z_t;\alpha_t^{\text{MDLM}}\mathbf x + (1-\alpha^{\text{MDLM}}_t)\mathbf m),\label{eq:marginal_mdm} 
\end{align}
where $\alpha_t^{\text{MDLM}} \in [0,1]$.
\paragraph{Forward Process.}
\begin{align}
    q(\mathbf z_t\mid\mathbf z_s ) = \text{Cat}\left(\mathbf z_t;\alpha^{\text{MDLM}}_{t|s}\mathbf z_s + (1-\alpha^{\text{MDLM}}_{t|s})\mathbf m \right),
    \label{eq:mdlm-forward}
\end{align}
where $\alpha_{t|s} = \alpha_t/\alpha_s$.
\paragraph{Backward Process.}
\begin{align}
p_\theta\left(\mathbf{z}_s \mid \mathbf{z}_t\right)\begin{cases}\text{Cat}\left(\mathbf{z}_s ; \mathbf{z}_t\right), & \mathbf{z}_t \neq \mathbf{m}, \\ \text{Cat}\left(\mathbf{z}_s ; \frac{\left(1-\alpha_s\right) \mathbf{m}+\left(\alpha^{\text{MDLM}}_s-\alpha^{\text{MDLM}}_t\right) \mathbf{x}_\theta\left(\mathbf{z}_t, t\right)}{1-\alpha^{\text{MDLM}}_t}\right) . & \mathbf{z}_t=\mathbf{m},\end{cases}
\end{align}
\paragraph{Training Loss.}
\begin{align}
\mathcal{L}^T_{\text{diffusion}}&=-T\mathbb E_t\mathbb{E}_{q}\left[\frac{\alpha_s^{\text{MDLM}}-\alpha_t^{\text{MDLM}}}{1-\alpha_t^{\text{MDLM}}}\log(\mathbf x_\theta^\top\mathbf x)\right]\label{eq:loss_mdm},&\text{discrete version}\\
\mathcal{L}^\infty_{\text{diffusion}}&=\mathbb E_t\mathbb{E}_{q}\left[\frac{(\alpha^{\text{MDLM}}_t)'}{1-\alpha_t^{\text{MDLM}}}\log (\mathbf x_\theta^\top\mathbf x)\right],\label{eq:loss_mdm_continuous} &\text{continuous version}
\end{align}

\begin{proposition}\label{pro:mdlm-ours}
If we set $\rho_t\equiv 1$ and $\gamma_t \equiv \alpha_t^{\text{MDLM}}$ in \eqref{eq:marginal}, SCDD can recover MDLM. In particular, the marginal distribution, forward process \& posterior, and loss function can recover MDLM. 
\end{proposition}

\begin{proof}
It's easy to see the marginal \eqref{eq:marginal} becomes \eqref{eq:marginal_mdm} and the forward process \eqref{eq:forward} becomes \eqref{eq:mdlm-forward} whenever $\rho_t\equiv 1$ and $\gamma_t \equiv \alpha_t^{\text{MDLM}}$. Therefore, the backward process of MDLM and SCDD also align based on Bayes' Rule and the  same model parameterization. 
Finally, we verify the training loss. If $\mathbf z_t=\mathbf m$, the diffusion loss \eqref{eq:loss_m} becomes 
\begin{align}
\mathcal{L}^T_{\text{diffusion}}\mid_{\mathbf z_t=\mathbf m}&=-T\mathbb E_t\mathbb E_q\sum_{\mathbf v\neq\mathbf m}\frac{\gamma_s-\gamma_t}{1-\gamma_t}(\rho_s\mathbf v^{\top}\mathbf x+(1-\rho_s)\tfrac{1}{K})\log (\rho_t\mathbf v^\top\mathbf x_\theta+(1-\rho_t)\tfrac{1}{K})\nonumber\\
&=-T\mathbb E_t\mathbb E_q\left[\frac{\gamma_s-\gamma_t}{1-\gamma_t}\log \mathbf x_\theta^\top\mathbf x\right].\nonumber 
\end{align}
If $\mathbf z_t\neq\mathbf m$, then $\mathbf z_t = \mathbf z_s=\mathbf x$. From \eqref{eq:loss_v}, we have 
\begin{align}
\mathcal{L}^T_{\text{diffusion}}\mid_{\mathbf z_t=\mathbf x}&=-T\mathbb E_t\mathbb E_q\left[\frac{\mathbf x^\top \mathbf z_s}{\mathbf x^\top \mathbf z_t }(\mathbf z_s^\top \mathbf z_t) \log \frac{\mathbf z_s^\top\mathbf x_\theta}{\mathbf z_t^\top\mathbf x_\theta}\right]\nonumber\\
&=-T\log 1 =0.\nonumber 
\end{align}
Therefore, we recover \eqref{eq:loss_mdm}.

For the continuous loss, if $\mathbf z_t\neq \mathbf{m}$,  \eqref{eq:loss_continuous} vanishes since $\rho'_t\equiv0$. If $\mathbf z_t=\mathbf{m}$, substituting $\rho_t\equiv1$ yields 
$$\mathcal L^\infty_{\text{diffusion}}\mid_{\mathbf z_t = \mathbf m}=\mathbb E_t\mathbb E_q\left[\frac{\gamma'_t}{1-\gamma_t}\log(\mathbf x^\top\mathbf x_\theta)\right],$$
which coincides with \eqref{eq:loss_mdm_continuous}. 
\end{proof}

\subsection{SCDD and GIDD}
\label{sec:scdd-gidd}
In this section, we discuss the relation between the SCDD and GIDD \cite{rutte2025generalized}, and derive the noise schedule of SCDD that recovers the marginal distribution of GIDD. Firstly, recall several facts of GIDD (see \citet{rutte2025generalized} for full derivations):

\paragraph{Marginal Distribution.}
The marginal distribution at time $t\in[0,1]$ is
\begin{equation}
q^{\text{GIDD}}_t(\mathbf z_t \mid\mathbf x)
=
\text{Cat}\left(\mathbf z_t; \alpha_t^{\text{GIDD}}\mathbf x + \beta_t^{\text{GIDD}} \pi_t\right),
\label{eq:gidd-marginal}
\end{equation}
where $\beta_t^{\text{GIDD}} = 1 - \alpha^{\text{GIDD}}_t$ and $\alpha^{\text{GIDD}}_t \in [0,1]$.

\paragraph{Forward Process.}
Given a mixing rate $\alpha_t$ and a time-dependent mixing distribution $\pi_t$, the forward Markov chain is defined such that the cumulative transition from $s$ to $t$ satisfies
\begin{equation}
q^{\text{GIDD}}_{t|s}(\mathbf z_t \mid \mathbf z_s)
=
\text{Cat}\left(\mathbf z_t; \alpha_{t|s}^{\text{GIDD}} \mathbf z_s + \beta_{t|s}^{\text{GIDD}} \pi_{t|s}\right),
\label{eq:gidd-forward}
\end{equation}
where
\begin{equation}
\alpha_{t|s}^{\text{GIDD}} = \frac{\alpha_t^{\text{GIDD}}}{\alpha_s^{\text{GIDD}}}, 
\qquad
\beta_{t|s}^{\text{GIDD}} \pi_{t|s} = \beta_t^{\text{GIDD}} \pi_t - \frac{\alpha_t^{\text{GIDD}}}{\alpha_s^{\text{GIDD}}} \beta_s^{\text{GIDD}} \pi_s.\label{eq:gidd-params}
\end{equation}
\paragraph{Forward Transition Rate}
The forward rate matrix of GIDD is given by
\begin{align}
  R_t^{\text{GIDD}}(\mathbf z_s,\mathbf z_t)=\frac{(\alpha^{\text{GIDD}}_t)'}{\alpha^{\text{GIDD}}_t}\delta_{\mathbf z_s,\mathbf z_t}+\mathbf z_t^\top(\beta^{\text{GIDD}}_t \pi_t'-\frac{(\alpha^{\text{GIDD}}_t)'}{\alpha^{\text{GIDD}}_t}\pi_t),\label{eq:gidd-rate}  
\end{align}
and thus 
$$q^{\text{GIDD}}_{t|s}(\mathbf z_t|\mathbf z_s)=\delta_{\mathbf z_s,\mathbf z_t}+R^{\text{GIDD}}_t(\mathbf z_s,\mathbf z_t)\Delta t+o(\Delta t),$$
where $\Delta t = t-s$.

\paragraph{Backward process.}
The backward process is parameterized as
\begin{equation}
p^{\text{GIDD}}_\theta(\mathbf z_s \mid \mathbf z_t)
=
\frac{q^{\text{GIDD}}_{t|s}(\mathbf z_t \mid \mathbf z_s) q^{\text{GIDD}}_s(\mathbf z_s \mid\mathbf x_\theta)}{q^{\text{GIDD}}_t(\mathbf z_t \mid\mathbf x_\theta)}.
\label{eq:gidd-backward}
\end{equation}

\paragraph{Training Loss.}
The continuous-time negative ELBO is given by
\begin{equation}
-\log q(x)
\le
\mathbb{E}_{t,\mathbf z_t}
\left[
w_t(\mathbf z_t,x)
\big(
D_{\mathrm{KL}}(q^{\text{GIDD}}_t(\cdot \mid\mathbf x)\|q^{\text{GIDD}}_t(\cdot \mid\mathbf x_\theta))
+
D_{\mathrm{IS}}(q^{\text{GIDD}}_t(\mathbf z_t \mid\mathbf x)\|q^{\text{GIDD}}_t(\mathbf z_t \mid\mathbf x_\theta))
\big)
\right]
+
C,
\label{eq:gidd-loss}
\end{equation}
with weighting term
\begin{equation}
w_t(\mathbf z_t,\mathbf x)
=
\frac{1}{q^{\text{GIDD}}_t(\mathbf z_t \mid\mathbf x)}
\mathbf z_t^\top
\left(
\beta_t^{\text{GIDD}} \pi'_t
-
\frac{(\alpha^{\text{GIDD}}_t)'}{\alpha_t^{\text{GIDD}}}\pi_t
\right),
\label{eq:gidd-weight}
\end{equation}
where $D_{\text{IS}}$ denotes the Itakura--Saito divergence and $C$ is the ELBO constant.

The next proposition shows that GIDD can be rewritten under our marginal parameterization, but with a non-absorbing \textsc{[mask]} token and more coupled forward transition kernel.
\begin{proposition}[Reparameterization of GIDD]\label{pro:gidd-ours}
Consider the forward transition kernel
\begin{align}
q(\mathbf z_t|\mathbf z_s)
&= \mathrm{Cat}\Bigl(
\mathbf z_t;
\gamma_{t\mid s}\rho_{t\mid s}\mathbf z_s
+\gamma_t(1-\rho_{t\mid s})\mathbf u
+\bigl((1-\gamma_t)-\gamma_{t\mid s}\rho_{t\mid s}(1-\gamma_s)\bigr)\mathbf m
\Bigr).
\label{eq:gidd_forward2}
\end{align}
This kernel induces the marginal distribution in \eqref{eq:marginal}.
Its continuous-time transition rate is
\begin{align}
R_t(\mathbf z_t,\mathbf z_s)
&=
\left(\frac{\rho'_t}{\rho_t}+\frac{\gamma_t'}{\gamma_t}\right)
\delta_{\mathbf z_s,\mathbf z_t}
-\mathbf z_t^{\top}
\left[
\gamma_t \frac{\rho'_t}{\rho_t}\mathbf u
+
\left((1-\gamma_t)\frac{\rho_t'}{\rho_t}
+\frac{\gamma_t'}{\gamma_t}\right)\mathbf m
\right].
\label{eq:gidd-rate2}
\end{align}
Moreover, under the reparameterization
\begin{equation}
\begin{cases}
    \alpha_t^{\mathrm{GIDD}} = \rho_t \gamma_t, \\
    \beta_t^{\mathrm{GIDD}} = 1 - \alpha_t^{\mathrm{GIDD}},\\
    \beta_t^{\mathrm{GIDD}} \pi_t
    =
    \gamma_t(1 - \rho_t) \mathbf u
    + (1 - \gamma_t) \mathbf m,
\end{cases}
\label{eq:giddreparam}
\end{equation}
the forward process induced by \eqref{eq:gidd_forward2} coincides with
the GIDD forward process in \eqref{eq:gidd-forward}. Consequently, the
marginal distribution and continuous-time transition rate induced by
\eqref{eq:gidd_forward2} coincide with the GIDD marginal distribution and
transition rate in \eqref{eq:gidd-marginal} and \eqref{eq:gidd-rate},
respectively.
\end{proposition}
\begin{remark}
Proposition~\ref{pro:gidd-ours} shows that the forward kernel in
\eqref{eq:gidd_forward2} provides a reparameterized form of the GIDD
forward process under the parameter translation in
\eqref{eq:giddreparam}. This kernel induces the same marginal
distribution as \eqref{eq:marginal}, but unlike the SCDD forward kernel,
it does not impose the absorbing-mask condition
$q(\mathbf z_t=\mathbf m\mid \mathbf z_s=\mathbf m)=1$.
Therefore, SCDD and GIDD can be expressed under a common marginal
parameterization, while differing in their forward transition kernels.
This distinction explains why the uniform-transition and masking
components are coupled in the GIDD transition rate, whereas they are
controlled separately in SCDD.
\end{remark}

\begin{proof}
We prove the proposition in three steps.

\textbf{Step 1: Marginal distribution.}
We first verify that the transition kernel in \eqref{eq:gidd_forward2}
induces the marginal distribution in \eqref{eq:marginal}. Suppose that
at time \(s\),
\[
q(\mathbf z_s\mid \mathbf x)
=
\mathrm{Cat}\Bigl(
\mathbf z_s;
\gamma_s\bigl(\rho_s\mathbf x+(1-\rho_s)\mathbf u\bigr)
+(1-\gamma_s)\mathbf m
\Bigr).
\]
By the forward transition kernel \eqref{eq:gidd_forward2}, the marginal
at time \(t\) has probability vector
\begin{align}
&\gamma_{t\mid s}\rho_{t\mid s}
\Bigl[
\gamma_s\bigl(\rho_s\mathbf x+(1-\rho_s)\mathbf u\bigr)
+(1-\gamma_s)\mathbf m
\Bigr]
+\gamma_t(1-\rho_{t\mid s})\mathbf u  \notag\\
&\qquad
+\Bigl((1-\gamma_t)
-\gamma_{t\mid s}\rho_{t\mid s}(1-\gamma_s)\Bigr)\mathbf m.
\end{align}
Using \(\gamma_{t\mid s}=\gamma_t/\gamma_s\) and
\(\rho_{t\mid s}=\rho_t/\rho_s\), the coefficient of \(\mathbf x\) is
\[
\gamma_{t\mid s}\rho_{t\mid s}\gamma_s\rho_s
=
\gamma_t\rho_t.
\]
The coefficient of \(\mathbf u\) is
\begin{align}
\gamma_{t\mid s}\rho_{t\mid s}\gamma_s(1-\rho_s)
+\gamma_t(1-\rho_{t\mid s})
&=
\gamma_t\rho_t\frac{1-\rho_s}{\rho_s}
+\gamma_t\left(1-\frac{\rho_t}{\rho_s}\right)  \notag\\
&=
\gamma_t(1-\rho_t).
\end{align}
The coefficient of \(\mathbf m\) is
\[
\gamma_{t\mid s}\rho_{t\mid s}(1-\gamma_s)
+
(1-\gamma_t)
-\gamma_{t\mid s}\rho_{t\mid s}(1-\gamma_s)
=
1-\gamma_t.
\]
Therefore,
\[
q(\mathbf z_t\mid \mathbf x)
=
\mathrm{Cat}\Bigl(
\mathbf z_t;
\gamma_t\bigl(\rho_t\mathbf x+(1-\rho_t)\mathbf u\bigr)
+(1-\gamma_t)\mathbf m
\Bigr),
\]
which is exactly \eqref{eq:marginal}.

\textbf{Step 2: Forward transition rate.}
Let \(t=s+\Delta t\). From \eqref{eq:gidd_forward2}, the probability that
\(\mathbf z_t\) retains \(\mathbf z_s\) is
\[
\gamma_{t\mid s}\rho_{t\mid s}
=
\frac{\gamma_t\rho_t}{\gamma_s\rho_s}.
\]
Using first-order Taylor expansions,
\[
\frac{\gamma_t}{\gamma_s}
=
1+\frac{\gamma_t'}{\gamma_t}\Delta t+o(\Delta t),
\qquad
\frac{\rho_t}{\rho_s}
=
1+\frac{\rho_t'}{\rho_t}\Delta t+o(\Delta t),
\]
we obtain
\[
\gamma_{t\mid s}\rho_{t\mid s}
=
1+
\left(
\frac{\gamma_t'}{\gamma_t}
+
\frac{\rho_t'}{\rho_t}
\right)\Delta t
+o(\Delta t).
\]
Similarly, the coefficient of the uniform transition component is
\[
\gamma_t(1-\rho_{t\mid s})
=
-\gamma_t\frac{\rho_t'}{\rho_t}\Delta t
+o(\Delta t),
\]
and the coefficient of the masking component is
\begin{align}
(1-\gamma_t)
-\gamma_{t\mid s}\rho_{t\mid s}(1-\gamma_s)
&=
-\left[
(1-\gamma_t)\frac{\rho_t'}{\rho_t}
+
\frac{\gamma_t'}{\gamma_t}
\right]\Delta t
+o(\Delta t).
\end{align}
Thus,
\[
q(\mathbf z_t\mid \mathbf z_s)
=
\delta_{\mathbf z_s,\mathbf z_t}
+
\Delta t\,R_t(\mathbf z_t,\mathbf z_s)
+
o(\Delta t),
\]
where
\[
R_t(\mathbf z_t,\mathbf z_s)
=
\left(\frac{\rho'_t}{\rho_t}+\frac{\gamma_t'}{\gamma_t}\right)
\delta_{\mathbf z_s,\mathbf z_t}
-\mathbf z_t^{\top}
\left[
\gamma_t \frac{\rho'_t}{\rho_t}\mathbf u
+
\left((1-\gamma_t)\frac{\rho_t'}{\rho_t}
+\frac{\gamma_t'}{\gamma_t}\right)\mathbf m
\right],
\]
which proves \eqref{eq:gidd-rate2}.

\textbf{Step 3: Equivalence with GIDD.}
It remains to connect the forward process in \eqref{eq:gidd_forward2}
with the original GIDD parameterization. Under the reparameterization
\[
\alpha_t^{\mathrm{GIDD}}=\rho_t\gamma_t,
\qquad
\beta_t^{\mathrm{GIDD}}=1-\alpha_t^{\mathrm{GIDD}},
\qquad
\beta_t^{\mathrm{GIDD}}\pi_t
=
\gamma_t(1-\rho_t)\mathbf u+(1-\gamma_t)\mathbf m,
\]
we have
\[
\alpha_{t\mid s}^{\mathrm{GIDD}}
=
\frac{\alpha_t^{\mathrm{GIDD}}}{\alpha_s^{\mathrm{GIDD}}}
=
\frac{\rho_t\gamma_t}{\rho_s\gamma_s}
=
\rho_{t\mid s}\gamma_{t\mid s}.
\]
Moreover,
\begin{align}
\beta_t^{\mathrm{GIDD}}\pi_t
-
\alpha_{t\mid s}^{\mathrm{GIDD}}
\beta_s^{\mathrm{GIDD}}\pi_s
&=
\gamma_t(1-\rho_t)\mathbf u
+(1-\gamma_t)\mathbf m  \notag\\
&\quad
-\gamma_{t\mid s}\rho_{t\mid s}
\Bigl[
\gamma_s(1-\rho_s)\mathbf u
+(1-\gamma_s)\mathbf m
\Bigr] \notag\\
&=
\gamma_t(1-\rho_{t\mid s})\mathbf u
+
\bigl((1-\gamma_t)
-\gamma_{t\mid s}\rho_{t\mid s}(1-\gamma_s)\bigr)\mathbf m.
\end{align}
Therefore, substituting \eqref{eq:giddreparam} into the GIDD forward
kernel \eqref{eq:gidd-forward} gives
\[
\alpha_{t\mid s}^{\mathrm{GIDD}}\mathbf z_s
+
\beta_t^{\mathrm{GIDD}}\pi_t
-
\alpha_{t\mid s}^{\mathrm{GIDD}}
\beta_s^{\mathrm{GIDD}}\pi_s
=
\gamma_{t\mid s}\rho_{t\mid s}\mathbf z_s
+\gamma_t(1-\rho_{t\mid s})\mathbf u
+\bigl((1-\gamma_t)
-\gamma_{t\mid s}\rho_{t\mid s}(1-\gamma_s)\bigr)\mathbf m,
\]
which is exactly the probability vector in \eqref{eq:gidd_forward2}.
Hence, the forward process induced by \eqref{eq:gidd_forward2}
coincides with the GIDD forward process. Consequently, the marginal
distribution and continuous-time transition rate induced by
\eqref{eq:gidd_forward2} coincide with the GIDD marginal distribution
and transition rate in \eqref{eq:gidd-marginal} and
\eqref{eq:gidd-rate}, respectively.
\end{proof}

\begin{remark}\label{rm:gidd-ours} [Comparison of Forward Transition Rates between SCDD and GIDD.]

Under CTMC theory and using the same parameterization of marginal distribution, the GIDD's forward transition rate is given by: 
$$R^{\text{GIDD}}_t(\mathbf z_s,\mathbf z_t)=\left(\frac{\gamma_t'}{\gamma_t}+\frac{\rho'_t}{\rho_t}\right)\mathbf z_s^\top\mathbf z_t-\mathbf z_t^\top\left(\gamma_t\frac{\rho_t'}{\rho_t}\mathbf u+\left((1-\gamma_t)\frac{\rho_t'}{\rho_t} + \frac{\gamma_t'}{\gamma_t}\right)\mathbf m\right), $$
and SCDD's transition rate, when $\mathbf z_s \neq \mathbf m$, is given by equation \eqref{eq:forward_rate}: 

$$R^{\text{SCDD}}_t(\mathbf z_s,\mathbf z_t)=\left(\frac{\gamma_t'}{\gamma_t}+\frac{\rho'_t}{\rho_t}\right)\mathbf z_s^\top \mathbf z_t-\mathbf z_t^\top\left(\frac{\rho_t'}{\rho_t}\mathbf u+\frac{\gamma_t'}{\gamma_t}\mathbf m\right).$$

Our formulation yields a simpler forward process by making $\mathbf m$ an absorbing state. In particular, the uniform transition noise $\mathbf u$ and the absorbing mask noise $\mathbf m$ become \textbf{decoupled}, and their rates are independently controlled by the marginal parameters $\rho_t$ and $\gamma_t$, respectively. In contrast, GIDD couples these two noise channels, making the parameters' respective effects less explicit. This decoupling makes the forward dynamics more interpretable and substantially reduces the algebraic complexity of the backward formulas and the resulting training loss, which in turn simplifies implementation and lowers maintenance cost in practice.
\end{remark}

\section{Experimental Details}
\subsection{Training Details.}\label{sec:training-details}
For fair comparison, we retrained MDLM \citep{sahoo2024simple} and GIDD \citep{rutte2025generalized} along with our SCDD model. The detailed hyper-parameter setting for different datasets can be found in Table \ref{tab:param-config}. We train MDLM and SCDD with a diffusion process and loss function defined over $T=1000$ discrete time steps, and train GIDD using the continuous-time loss with dynamic weighting to guarantee the best performance, see \citep{rutte2025generalized} for the impact of weighting function on GIDD training. We also align the marginal distribution of SCDD with GIDD to guarantee two models see distributionally equivalent samples during training. Both of the models are trained using a noise schedule that achieves uniform transition peak ratio of $p_u\in\{0.1,0.2\}$ at $t=0.5$. See Appendix \ref{sec:noise_schedule_scdd} for detailed discussion.  

All the models use the DiT \cite{peebles2023scalable} structure as the backbone of denoising networks, and GPT-2 tokenizer \cite{radford2019language}. Following MDLM and GIDD, we use the \textsc{small} variant of DiT as the denoising network with 12 Transformer blocks, 12 attention heads per block, and 768 hidden dimension, yielding 166M trainable parameters. Time conditioning is projected to a dimension of 128 before being injected into the network. No dropout is applied during training. 

For LM1B \cite{chelba2013onebillion} dataset, all the model are trained with a batch size of 512, context length of 128, for 500k steps, yielding 33B ($512\times128\times500$k) training tokens in total. For OWT \cite{Gokaslan2019OpenWeb} dataset, all the models are trained with batch size of 256, context length of 512, for 1M steps, yielding 131B ($256\times512\times1$M) training tokens in total. We leave the last 100k docs as validation as in MDLM training. For ablation study, we train the model on Wikitext-103 \cite{merity2017pointer} with a batch size of 128, context length 512, for 100k steps, yielding 7B ($128\times512\times100$k) training tokens in total. 

For optimization, we use the AdamW optimizer \cite{adam2014method} with hyperparameters $\beta_1=0.9$, $\beta_2=0.999$, and $\epsilon=1\text{e}-9$. The learning rate increases linearly from $1 \times 10^{-6}$ to the peak rate $5 \times 10^{-4}$  over the first 10k steps, then follows a cosine-decay schedule to $5 \times 10^{-5}$ for the remainder of the training. To stabilize training, we use \texttt{bfloat16} mixed precision and maintain an Exponential Moving Average of model parameters with a decay rate of $0.9999$. Training are completed on a single node of 4 (Wikitext-103) or 8 (LM1B and OWT) NVIDIA H100-80GB GPUs.

\begin{table}[!ht]
    \centering
     \caption{Hyper-parameter configuration for each dataset.}
    \label{tab:param-config}
    \begin{tabular}{r|rrr}
    \toprule
         & Wikitext-103 & LM1B & OWT \\
    \midrule
    Context Length & 512 & 128 & 512  \\
    Training Steps $N$ & 100K & 500K & 1M \\
    Batch Size & 128 & 512 & 256 \\
    Number of GPUs & 4 & 8 & 8 \\
    Learning Rate & \multicolumn{3}{c}{5e-4} \\
    T  & \multicolumn{3}{c}{1000}\\

    \midrule
    Training and testing hardware  &  \multicolumn{3}{c}{H100-80GB GPUs} \\
    \bottomrule
    \end{tabular}
   
\end{table}

\label{sec:experiments}
\subsection{Unconditional Text Generation Details.}
\label{sec:addexp}
As noted by \citet{zheng2025masked}, 32-bit floating-point Gumbel-max categorical sampling suffers from numerical precision issues. Therefore, we use 64-bit floating-point for all text generations. ReMDM-cap and ReMDM-conf \cite{wang2025remasking} are two customized samplers that can be readily applied to pretrained MDLM model to elicit self-correction. For ReMDM-cap we use the default $\eta_{cap}=0.01$. We apply nucleus sampling \cite{Holtzman2020The} with $p=0.9$ throughout the experiments as it improves the quality of generated texts.

\paragraph{Entropy Results.}
We report the unigram entropy results in Table \ref{tab:entropy-comparison} in addition to the generative perplexity results as a sanity check. 
\begin{table*}[!ht]
\caption{Entropy on LM1B and OWT datasets across sampling steps. Two decimal places are shown.}
\label{tab:entropy-comparison}
\vskip 0.15in
\begin{center}
\begin{small}
\begin{sc}
\renewcommand{\arraystretch}{0.85}
\setlength{\tabcolsep}{5pt}
\begin{tabular}{lccccccccccc}
\toprule
\qquad Entropy & \multicolumn{5}{c}{LM1B (Steps)} & \multicolumn{6}{c}{OWT (Steps)} \\
\cmidrule(r){2-6} \cmidrule(l){7-12}
 & 16 & 32 & 64 & 128 & 256 & 32 & 64 & 128 & 256 & 512 & 1024 \\
\midrule
MDLM & 4.37 & 4.36 & 4.36 & 4.36 & 4.36 & 5.26 & 5.23 & 5.21 & 5.19 & 5.19 & 5.17 \\[2pt]
ReMDM Cap 0.01 & 4.37 & 4.36 & 4.36 & 4.36 & 4.35 & 5.26 & 5.23 & 5.19 & 5.17 & 5.15 & 5.13 \\[2pt]
ReMDM Confidence & 4.37 & 4.36 & 4.35 & 4.35 & 4.36 & 5.26 & 5.22 & 5.19 & 5.17 & 5.15 & 5.13 \\[2pt]
GIDD ($p_u=0.1$) & 4.34 & 4.35 & 4.35 & 4.36 & 4.36 & 5.09 & 5.07 & 5.06 & 5.06 & 5.05 & 5.05 \\[2pt]
GIDD ($p_u=0.2$) & 4.34 & 4.35 & 4.36 & 4.36 & 4.36 & 5.09 & 5.08 & 5.07 & 5.07 & 5.06 & 5.06 \\[2pt]
SCDD ($p_u=0.1$) & 4.21 & 4.23 & 4.24 & 4.24 & 4.24 & 4.83 & 4.86 & 4.86 & 4.86 & 4.86 & 4.85 \\[2pt]
SCDD ($p_u=0.2$) & 4.22 & 4.24 & 4.24 & 4.25 & 4.25 & 4.85 & 4.87 & 4.86 & 4.87 & 4.86 & 4.85 \\
\bottomrule
\end{tabular}
\end{sc}
\end{small}
\end{center}
\vskip -0.1in
\end{table*}
\paragraph{LLM-as-a-judge.} In the LLM-as-a-judge experiment, we borrow the LLM prompt from \citet{rutte2025generalized}, as in Figure \ref{fig:evaluation_prompt}. 
\begin{figure*}[t]
\begin{tcolorbox}[
    colback=gray!5, 
    colframe=gray!50, 
    title=\textbf{Evaluation Prompt},
    fonttitle=\bfseries,
    sharp corners
]
\small
\texttt{1. Clarity and coherence: Keeping in mind that the text may be cut off in the beginning and at the end due to it being an excerpt, how clear and understandable is the text? \\ 2. Grammaticality: Are there any grammatical errors in the text? \\ 3. Factuality: If applicable, is the factually verifiable information stated in the text (e.g. facts about geography, history, etc.) accurate and reliable? \\ 4. Writing style: How well is the text written in terms of style and fluency? Do the sentences flow well, is the vocabulary appropriate? \\ 5. Creativity: How original and creative is the text? \\ For each category, give a short justification before providing the final score. Your answer should be following the JSON format, with one top-level key for each aspect (‘ clarity‘, ‘grammaticality‘, ‘factuality‘, ‘style‘, and ‘creativity‘). Each aspect, in turn, should be a JSON object consisting of a ‘reasoning‘ and ‘score‘ key in that order. The ‘reasoning‘ key should contain a short justification for the score, and the ‘score‘ key should contain the score itself. \\ Please keep the following in mind: - Give your justification first before deciding on a final score. - Only output the JSON containing the justifications and scores and nothing else. - Keep in mind that the presented paragraph may be an excerpt from a longer document, so it may not be fully self-contained. Do not deduct points for issues arising from this. \\ The text to be graded is as follows: ‘‘‘ {text} ‘‘‘
}
\end{tcolorbox}
\caption{Evaluation Prompt used for LLM-as-a-judge experiment. }
\label{fig:evaluation_prompt}
\end{figure*}

Beyond the matched-schedule setting, we also evaluate SCDD and GIDD+ under a cross-ratio setting by comparing the best-performing SCDD ($p_u=0.2$) with the best GIDD+ ($p_u=0.1$). In this setup, clean OWT sequences are still corrupted using SCDD’s forward noising process; consequently, GIDD+ ($p_u=0.1$) must reconstruct inputs that are ``dirtier'' than those seen during its original training. Table \ref{tab:crossratio_compact} shows that SCDD consistently maintains statistically significant better performance in \textit{Clarity}, \textit{Factuality}, and \textit{Style} across all sampling steps. Despite GIDD+ showing an advantage in \textit{Creativity}, SCDD achieves a higher overall win rate. These results confirm that GIDD+ ($p_u=0.1$), which is trained at a lower noise level, lacks the capacity to effectively correct texts at higher corruption levels.
\begin{table*}[t]

\caption{Cross-ratio Setting --- SCDD ($p_u{=}0.2$) vs GIDD+ ($p_u{=}0.1$). Values are formatted as SCDD (GIDD+). Significance: $^{*} p < 0.05$, $^{**} p < 0.01$.}
\label{tab:crossratio_compact}
\begin{center}
\begin{sc}
\renewcommand{\arraystretch}{0.8}
\setlength{\tabcolsep}{3.5pt} 
\begin{tabular}{lcccccc}
\toprule
Metrics & \multicolumn{6}{c}{Steps} \\
\cmidrule(l){2-7}
 & 32 & 64 & 128 & 256 & 512 & 1024 \\
\midrule
Clarity & 1.64 (1.52)$^{*}$ & 1.66 (1.51)$^{*}$ & 1.70 (1.46)$^{**}$ & 1.70 (1.54)$^{**}$ & 1.73 (1.54)$^{**}$ & 1.73 (1.52)$^{**}$ \\[2pt]
Gramm. & 1.39 (1.48) & 1.45 (1.51) & 1.49 (1.44) & 1.52 (1.51) & 1.54 (1.53) & 1.51 (1.49) \\[2pt]
Fact. & 2.25 (2.13)$^{**}$ & 2.16 (2.07)$^{*}$ & 2.16 (2.01)$^{**}$ & 2.22 (2.08)$^{**}$ & 2.20 (2.05)$^{**}$ & 2.12 (2.01)$^{*}$ \\[2pt]
Style & 1.64 (1.54) & 1.65 (1.54)$^{*}$ & 1.67 (1.49)$^{**}$ & 1.67 (1.55)$^{*}$ & 1.71 (1.57)$^{*}$ & 1.72 (1.52)$^{**}$ \\[2pt]
Creativity & 2.86 (3.21)$^{**}$ & 2.95 (3.27)$^{**}$ & 2.91 (3.25)$^{**}$ & 2.93 (3.26)$^{**}$ & 2.98 (3.31)$^{**}$ & 2.98 (3.27)$^{**}$ \\[2pt]
\midrule
Win rate & 52.4\% & 53.9\% & 57.8\%$^{*}$ & 56.0\% & 55.1\% & 58.0\%$^{*}$ \\
\bottomrule
\end{tabular}
\end{sc}
\end{center}
\end{table*}

\subsection{Benchmark Performance.}
\label{sec:ben}
For each zero-shot task, we rank candidate completions by estimating the conditional log-likelihood of the answer $y$ given the question $x$. Because diffusion models are non-autoregressive, we score each candidate using the ELBO of the concatenated sequence $[x; y]$. Since $P(x)$ is identical across candidates, maximizing $\log P(x,y)$ is equivalent to maximizing $\log P(y \mid x)$.

Concretely, for each $[x; y]$ pair, we compute its contribution to the diffusion variational loss, which is stochastic due to sampling $t \sim \mathcal{U}(0,1)$, and the corresponding forward noising process. We therefore use a Monte Carlo estimate, averaging the ELBO over $N=10$ independent forward passes, and select the candidate with the highest mean ELBO. At last, we calculate the accuracy of selection. Results are reported in Table \ref{tab:benchmark_accuracy}.

Since each of zero-shot tasks ranks answer candidates by their estimated ELBO, there is no self-correction during evaluation, making these benchmarks a measure of zero-shot likelihood rather than language generation. Uniform-noise diffusion models (including both SCDD and GIDD+) typically exhibit worse likelihood but stronger abilities in self-correction, few-step sampling, and generation tasks \cite{sahoo2026scaling}. Our generative PPL results (Table \ref{tab:genppl-comparison}) confirm that SCDD produces higher-quality text than MDLM/GIDD at equal sampling steps, demonstrating that worse likelihood estimation does not translate to weaker language generation.
\begin{table}[h]
  \caption{Zero-shot accuracy on various language modeling benchmarks. Comparison between MDLM, GIDD, and SCDD at different noise levels ($p_u$).}
  \label{tab:benchmark_accuracy}
  \begin{center}
    \begin{small}
      \begin{sc}
        \begin{tabular}{lcccccccc}
          \toprule
          Model & ARC-e & ARC-c & BoolQ & Hellaswag & OBQA & PIQA & WinoG & Avg. \\
          \midrule
          MDLM & 27.90 & 21.16 & 47.25 & 27.76 & 19.20 & 52.01 & 50.43 & 35.10 \\
          GIDD ($p_u=0.1$) & 26.85 & 22.01 & 48.99 & 26.36 & 15.60 & 50.76 & 50.27 & 34.41 \\
          GIDD ($p_u=0.2$) & 26.43 & 21.92 & 48.29 & 26.55 & 15.80 & 51.09 & 50.51 & 34.37 \\
          SCDD ($p_u=0.1$, ours) & 26.64 & 22.01 & 48.13 & 26.73 & 17.20 & 50.05 & 48.78 & 34.22 \\
          SCDD ($p_u=0.2$, ours) & 26.52 & 24.16 & 46.64 & 26.20 & 20.40 & 49.67 & 49.01 & 34.66 \\
          \bottomrule
        \end{tabular}
      \end{sc}
    \end{small}
  \end{center}
  \vskip -0.1in
\end{table}
\subsection{Noise Schedule of SCDD}
\label{sec:noise_schedule_scdd}
\citet{rutte2025generalized} define the following time-dependent rates to elicit self-correction:
\begin{align}
\alpha_t^{\text{GIDD}}&=\frac{1-t}{C_t},
\qquad
\beta_t^{\text{GIDD}}\pi_t=\frac{t}{C_t}\mathbf m+\frac{c_t}{C_t}\mathbf u,
\label{eq:gidd-noise}\\
\text{where }c_t&=Bt^{\gamma/2}(1-t)^{\gamma/2},
\qquad
C_t=1+c_t,\qquad B,\gamma>0\nonumber
\end{align}
Notably, the marginal probability of uniform transition noise ratio is a symmetric function of $t$, vanishing at endpoints $t=0,1$ and attaining its maximum at $t=1/2$. 
In \citet{rutte2025generalized}, this maximum ratio is denoted by $p_u$, and can be expressed as
$$
p_u
=
\frac{c_{1/2}}{C_{1/2}}=\frac{c_{1/2}}{1+c_{1/2}}.
$$ 
If $p_u$ is fixed, we can solve for the corresponding constant $B$ as:
$$
B
=
\frac{c_{1/2}}{(1/2)^{\gamma}}
=
2^{\gamma}\frac{p_u}{1-p_u}.
$$

To ensure a fair comparison between SCDD and GIDD, we align the marginal distribution of SCDD with that of GIDD in our experimental evaluation. By applying the parameter translation in \eqref{eq:giddreparam}, we obtain:
\begin{align}
\begin{cases}
\rho_t\gamma_t&=\dfrac{1-t}{1+c_t},\nonumber\\[6pt]
(1-\rho_t)\gamma_t&=\dfrac{c_t}{1+c_t},\nonumber \\[6pt]
1-\gamma_t&=\dfrac{t}{1+c_t}.\nonumber 
\end{cases}
\end{align}
Solve the above equations to obtain noise schedule for SCDD in experiments: 
\begin{align}
\begin{cases}
\gamma_t=\frac{1+c_t-t}{1+c_t},\\[4pt]
\rho_t=\frac{1-t}{1+c_t-t}, \\
\end{cases}\label{eq:noise-gidd-ours}
\end{align}
where $c_t:=2^{\gamma}\frac{p_u}{1-p_u} t^{\gamma/2}(1-t)^{\gamma/2},p_u\in\{0.1,0.2\}$.

In ablation study, we use a noise schedule that attains the maximum noise ratio at a general time point $t_{\text{peak}}$, where $B$ and $c_t$ are given as follows:
\begin{equation}
\begin{gathered}
B=\frac{p_u}{1-p_u} \cdot \frac{1}{t_{\text {peak}}^{\gamma t_{\text {peak}}}\left(1-t_{\text {peak}}\right)^{\gamma\left(1-t_{\text {peak}}\right)}} . \\
c_t=B t^{\gamma t_{\text {peak}}}(1-t)^{\gamma\left(1-t_{\text {peak}}\right)}=\frac{p_u}{1-p_u} \cdot \frac{t^{\gamma t_{\text {peak}}}(1-t)^{\gamma\left(1-t_{\text {peak}}\right)}}{t_{\text {peak}}^{\gamma t_{\text {peak}}}\left(1-t_{\text {peak}}\right)^{\gamma\left(1-t_{\text {peak}}\right)}} .
\end{gathered}
\end{equation}

\section{Related Works}
\label{sec:related_work}

\paragraph{Mask Diffusion Language Model (MDLM/MDM) \cite{sahoo2024simple,shi2024simplified}}
In vanilla masked diffusion models, the forward process is defined by a Markov chain of the form
$$
q(\mathbf z_t \mid \mathbf z_s)
=
\alpha_{t|s}\,\mathbf x
+
(1-\alpha_{t|s})\,\mathbf m,
$$
which induces a marginal distribution
$$
q(\mathbf z_t \mid \mathbf x)
=
\mathrm{Cat}\!\left(\mathbf z_t;\, \alpha_t \mathbf x + (1-\alpha_t)\mathbf m \right).
$$

Masked diffusion models offer significant computational advantages and substantially improve upon earlier discrete flow-matching approaches. However, they lack an explicit self-correction mechanism. In particular, once a token is decoded during inference, it becomes fixed and cannot be revised in subsequent denoising steps. As a result, early prediction errors accumulate over the backward process and harms generation quality. SCDD can be viewed as a generalization of MDLM to support self-correction. We refer to Appendix \ref{sec:mdlm-ours} for a formal comparison.

\paragraph{Remasking Diffusion Model (ReMDM) \cite{wang2025remasking}}

ReMDM introduces a collection of post-doc samplers that can be readily applied to pretrained MDLM models. To elicit self-correction, it adopts a \emph{remasking} mechanism in the backward (inference) process:
$$
q(\mathbf z_s \mid \mathbf z_t, \mathbf x)
=
\begin{cases}
\mathrm{Cat}\!\left(\mathbf z_s; (1-\sigma_t)\mathbf x + \sigma_t \mathbf m \right),\\
\mathrm{Cat}\!\left(\mathbf z_s;
\frac{\alpha_s-(1-\sigma_t)\alpha_t}{1-\alpha_t}\mathbf x
+
\frac{1-\alpha_s-\sigma_t\alpha_t}{1-\alpha_t}\mathbf m
\right),
\end{cases}
$$
where $\alpha_t, \sigma_t \in [0,1]$ are noise schedules. The non-Markovian forward process is derived from Bayes' rule.

While this remasking-based formulation enables token revision, it exhibits two notable limitations. First, ReMDM performs self-correction through a non-\textsc{[mask]} $\rightarrow$ \textsc{[mask]} $\rightarrow$ non-\textsc{[mask]} procedure, in which the intermediate masking step is redundant: it takes 2 steps to correct a token. This may reduce the parallel self-correction efficiency during inference. Second, the performance of ReMDM highly depends on the hyperparameters that affect the actual schedule of $\sigma_t$, which introduces additional tuning cost. 

\paragraph{Generalized Interpolating Discrete Diffusion (GIDD) \cite{rutte2025generalized}}
GIDD replaces the absorbing mask probability vector $\mathbf m$ in the marginal distribution of MDLM \cite{sahoo2024simple} with a more general corruption distribution $\pi_t$, defined as a mixture of absorbing mask $\mathbf m$ and uniform transition noise $\mathbf u$. The resulting marginal distribution takes the form
$$
q(\mathbf z_t \mid \mathbf x):=\mathrm{Cat}(\mathbf z_t;\alpha_t \mathbf x + (1-\alpha_t)\pi_t),
$$
which is induced by the following Markov transition:
$$
q(\mathbf z_t \mid \mathbf z_s):=
\mathrm{Cat}(\mathbf z_t;\alpha_{t|s}\mathbf z_s+(1-\alpha_t)\pi_t-\alpha_{t|s}(1-\alpha_s)\pi_s).
$$

SCDD and GIDD can both be viewed as generalizations of vanilla masked diffusion language models, differing primarily in their choices of the forward transition kernel. We refer to Proposition~\ref{pro:gidd-ours} and Remark~\ref{rm:gidd-ours} for a detailed comparison.

Compared to SCDD, GIDD exhibits the following limitations. First, GIDD controls uniform corruption and masking through an entangled parameterization, whereas SCDD decouples these two sources of noise. This decoupling leads to a substantially simpler backward process and a more tractable training loss. Second, GIDD doesn't eliminate the remasking step during inference, while SCDD supports direct token-to-token self-correction during both training and inference without introducing an intermediate masking stage.

\paragraph{Informed Corrector \cite{zhao2024informed}}
\citet{zhao2024informed} use continuous-time Markov chains (CTMCs) to describe the forward and backward processes:
$$
q_{t|s}(\mathbf y\mid\mathbf x)
=
\mathbf \delta_{\mathbf x,\mathbf y}
+
R_t(\mathbf x,\mathbf y)\Delta t
+
o(\Delta t),
$$
where $\mathbf x,\mathbf y\in\mathcal V^D$ denote token sequences (e.g. articles or paragraphs), and
$$
R_t(\mathbf x,\mathbf y)
=
\sum_{d=1}^D \beta_t R(\mathbf x_d,\mathbf y_d),
\qquad
R(\mathbf x_d,\mathbf y_d)
:=
\begin{cases}
1 & \mathbf y_d=\mathbf m,\ \mathbf x_d\neq\mathbf m,\\
-1 & \mathbf y_d=\mathbf x_d\neq\mathbf m,\\
0 & \mathbf x_d=\mathbf y_d=\mathbf m.
\end{cases}
$$

This formulation is equivalent to the forward Markov chain used in masked diffusion language models (MDLM), and the resulting marginal distribution can be written as
$$
q_t(\mathbf x_{t,d}\mid\mathbf x_{0,d})
=
\alpha_t\mathbf 1(\mathbf x_{t,d}=\mathbf x_{0,d})
+
(1-\alpha_t)\mathbf 1(\mathbf x_{t,d}=\mathbf m).
$$

During inference, \citet{zhao2024informed} introduce a \emph{confidence score} $c_d$ for each token position $\mathbf x_d$, defined for example as
$$
c_d
=
\log p_\theta(\mathbf x_{0,d}\mid\mathbf x_t^{/d})
-
\max_{i\neq d}\log p_\theta(\mathbf x_{0,i}\mid\mathbf x_t^{/d}),
$$
where $p_\theta(\cdot\mid\mathbf x_t^{/d})\approx q_{0|t}(\cdot\mid\mathbf x^{/d})$ denotes the denoising model, and $\mathbf x^{/d}$ is the sequence whose $d$-th component is removed or masked.

For positions with low confidence scores, the corrector step performs token-level resampling according to
$$
\mathbf x_d\sim p_\theta(\cdot^d\mid\mathbf x^{/d}).
$$

Unlike GIDD and SCDD, Informed Corrector enables self-correction only at the sampling or inference stage. The self-correction mechanism itself is not learned during training, but instead relies on model-dependent confidence estimates at inference time.

\paragraph{Plug-in Remasking for Inference-time Self-correction of Masked Diffusions (PRISM) \cite{kim2025fine}}

Given a pretrained masked diffusion model
$p_\theta$, PRISM trains a token-level score function $g_\phi$ via
$$
\mathcal L(\phi)
:=
\mathbb E_{\mathbf x,\mathbf z,i,\mathbf y\sim p_\theta^i(\cdot\mid\mathbf z)}
\Big[
\mathrm{CE}\big(\mathbf 1[\mathbf x_i=\mathbf y_i],g_\phi^i(\mathbf y)\big)
\Big],
$$
where $\mathrm{CE}$ denotes the cross-entropy loss.
Here $\mathbf x\sim p_{\mathrm{data}}$ denotes the original clean sequence,
$\mathbf z\sim q(\mathbf z\mid\mathbf x)$ is obtained by the forward masking process in MDLM, and $\mathbf y^i\sim p_\theta^i(\cdot\mid\mathbf z)$ is a token sampled from the pretrained MDM at a masked position $i$.

Intuitively, the score model $g_\phi^i(\mathbf y)$ estimates the conditional probability
$$
g_\phi^{i*}(\mathbf y)
=
\mathbb P(\mathbf x_i=\mathbf y_i\mid \mathbf y\oplus \mathbf m_i),
$$
where the randomness is induced by the joint distribution over
$\mathbf x\sim p_{\mathrm{data}}$,
$\mathbf z\sim q(\mathbf z\mid\mathbf x)$,
and $\mathbf y^i\sim p_\theta^i(\cdot\mid\mathbf z)$.
Here $\mathbf y\oplus \mathbf m_i$ denotes the sequence obtained by masking the $i$-th token of $\mathbf y$, indicating that $\mathbf y_i$ is unobserved when evaluating $g_\phi^i$.

During inference, at each intermediate state $\mathbf z_t$, the learned score
$g_\phi(\mathbf z_t)$ is used to identify unmasked tokens with low estimated quality.
Such tokens are remarked by setting $\mathbf z_{t,d}=\mathbf m$, enabling inference-time
self-correction through iterative remaking and unmasking.

\paragraph{Path
Planning Self-Planning (P2-Self) \cite{peng2025path}}
In the training stage, the authors follow the Masked Diffusion Language Model (MDLM) framework and train a denoising model
\begin{equation}
D_\theta: \mathcal{V}^L \to (\Delta^{d})^{L}.
\end{equation}
In inference, let
\begin{equation}
\mathbf{y} \sim D_\theta(\mathbf{x}_t).
\end{equation}
The authors introduce and train a planner
\begin{equation}
G_\phi:\mathcal{V}^{L}\times \mathcal{V}^{L} \to [0,1]^L,\qquad
(\mathbf{x}_t, \mathbf{y})\mapsto G_\phi(\mathbf{x}_t,\mathbf{y}),
\end{equation}
where $G_\phi^j(\mathbf{x}_t,\mathbf{y})$ denotes the probability that the $j$-th coordinate should be (re)sampled.
Given adjacent time steps $s<t$, the reverse update is defined by
\begin{equation}
i \sim \frac{G_\phi(\mathbf{x}_t,\mathbf{y})}{\sum_{j=1}^L G_\phi^j(\mathbf{x}_t,\mathbf{y})},
\end{equation}
and
\begin{equation}
q_{t,\theta}(\mathbf{x}_{s}^i \mid \mathbf{x}_t^i) =
\begin{cases}
\text{Cat}\!\left(\mathbf{x}_{s}^i;\dfrac{(1-\alpha_{t-1})\mathbf{m} + (\alpha_{t-1}-\alpha_t)D_\theta^i(\mathbf{x}_t)}{1-\alpha_t}\right), & \mathbf{x}_t^i = \mathbf{m},\\[1.2em]
\text{Cat}\!\left(\mathbf{x}_{s}^i;\dfrac{\big((\alpha_t-1)D_\theta^i(\mathbf{x}_t)^\top \mathbf{x}_t^i + 1-\alpha_{t-1}\big)\mathbf{x}_t^i + (\alpha_{t-1}-\alpha_t)D_\theta^i(\mathbf{x}_t)}{(1-\alpha_t)\big(1 - D_\theta^i(\mathbf{x}_t)^\top \mathbf{x}_t^i\big)}\right), & \mathbf{x}_t^i \neq \mathbf{m}.
\end{cases}
\end{equation}

Intuitively, $G_\phi(\mathbf{x}_t,\mathbf{y})$ selects a single coordinate $i$ to update.
If $\mathbf{x}_t^i = \mathbf{m}$, the transition reduces to the original unmasking rule in MDLM.
If $\mathbf{x}_t^i \neq \mathbf{m}$, the model remasks and resamples from $D_\theta$ at coordinate $i$ with some probability. 

Similar to PRISM and Informed Corrector, this method does not incorporate self-correction in the pre-training stage. In addition, there is inherent redundancy in the inference stage: If $\mathbf x_{t,i}$ is masked, the resampling mechanism will not work. 

\section{Unconditional Samples}
In this section, we present unconditional samples generated by SCDD ($p_u=0.2$) on OWT.
\paragraph{T=32}\texttt{...years, we have not seen an case where a human computer can read a wall, the size of a human brain. It is enough to make a perfectly accurate world map, if you are able to create a digital map, with the map information.\\
\\
The cameras are showing what they call "will", and how it will use it in the coming U of Indeed.e say the large number of Poles and soldiers expelled from Europe for being against the gold policy are going to do so, if will will be enough to them in their homes surrounded a themselves on the sea floor.\\
\\
For the same reason they they, and right, not.\\
\\
They told the vast canvass of research on the war is counter to the work of NATO's Chief Chiefsecurity, the Restipolar Prosperity Institute.\\
\\
They describe a room "deposition room" that has a room room of 1,000 feet', eyes to 1,000 flying words. The cameras in the same room away from the listening table of the central room. The teleponics are high telescopes and small spectrums, out to the walls, the "theerences" can be seen as points of reference.\\
\\
The cameras, they say, are't looking at people, of which the the information is stored, and are according to them, which no panoramic angles in order for information to settle in the small, confined space.\\
\\
They say that the methods of stalking and shooting at targets will be obsolete, too. This works wonders in the military. One Facebookmanulated me that "in the country, it's possible that an man will shoot them if they have his gun" and that cells are used in in our forces.\\
\\
They all began with arguments from young but gay, white anonymous authors. Professoricby said that the DUD can show that the man is, perhaps, "precompanceiveness of self-interest." They made a study that is "that they are just themselves victims are first first who be them," and that the are, for is are is come is the single target, but from what they are is do to them. However, they also offered a bias of the, such as "the difficulty of writing a novel with multiple readers," a novel about a young man living from home. And with the "mutual," the book will be...}

\paragraph{T=128}\texttt{...year, there will be about \$100 billion of \$1.50 wiped out by the dollar. The problem is that oil worth worth2 on the dollar and all of that that is there is banked by corruption that has a legal game to exist. And then there is the whole crisis that is, and the power of money and power. ...\\
\\
Take the financial addiction. Over the past two years, drug banks and the Justice Department have have collapsed at rates far, many times the world oil price. When we did it last year, we had back in it. We should be helping money on law money. Why not raise the taxes on crime because of money? We fail, we are out of the moral era of the big-money cono.\\
\\
But, it is not having the position to act a defendant in this case, said Sanjay Abdeh, the legal counsel at Steven Sinofsky, the chief in the case.\\
Buffalo's law is the first to allow for marijuana marijuana for anyone who wants to buy marijuana. It hasn't been tested yet but Attorney General John Hentry said anyone who buys marijuana and wants it is legally selling the market. But since Florida's existing marijuana possession has required background checks, the Attorney General said it would be able to allow in legal medical marijuana into the state.\\
\\
After the initial approval application for it last year, Gov. Rick Scott signed down a new law, allowing people able to possess the drug.\\
\\
The state is moving away the a laws that would allow the use of the legally Florida along with others of state circumstances. Under the new law, the for marijuana won't be allowed to marijuana or of criminal possession, so it means state- criminal background check if the defense to be a subject able will turn that to possession the Justice Department will do.\\
\\
The State Chief Information Officer John Betst said he is looking into the legal details as far as the amount and types what what is required and charges.\\
\\
Advertisement\\
\\
Florida's marijuana marijuana owner, Michelle Ibel, in Browesville County to meet and protect her husband and her three-year-old son. Ibel, a woman who lives from a war zone that has filled of violence, is believed to be one of the first people to be body-connell shot at a marijuana grow and medical center just off the street outside of the city's long-famous South City. Gang members are coming out and about. It is a new war...}

\paragraph{T=1024}\texttt{...on the different screen, that new features are going to come out, and the car is at least going of high promise. Stay tuned as we look forward to the future.\\
\\
It’s clear from the images that the product is much faster is it from Tesla. It will be rolling in two months of deliveries from now to mid-year. Below you can see the deliveries align into each other.\\
\\
When product is unveiled at E3 of this year it will be only (only) in the car’s first few months.\\
\\
By all accounts, the company’s product generation of products is well under the wings. Menting new companies are butnecessarily, products that have the attention of other players. It is also in into that major players not, such as the Apple ecosystem.\\
\\
On its side, Tesla could be a car car company. While Tesla has already worked with Uber and Lyft, it is going with the automaker with the direction creating an app where consumers will use better cars.\\
\\
The company is going down the technology in a new direction.\\
\\
"As Tesla brings together the same strengths and technology of the Prius software company, the result is better and better cars, more, electric cars, better performance on the road and better car safety and data."\\
\\
Tesla has seen since the with competitive advances, but those intangibles have, and last year its margins were even better.”[Source]\\
\\
Via Bloomberg\\
Donald Trump has been vocal about concerns about the "radical and media agenda" since to 9/11.\\
\\
But, that doesn’t mean Trump is going to address his concerns about the Trump Trump. It has been a lot of debate.\\
\\
In January, Donald Trump said, "I’s not to be great to you have to interviews in a pieces, a president’s opinion on these different issues” and made it a point that "the media word-for-word on the story I put out."\\
\\
Donna Trump asked a similar opinion. She asked Trump, "doesn’t Trump sound like a great leader of the country?" And his "strategy is winning, and he tells the truth understand." It’s a big question as the calls for the right to run for president. Over the past week, she has been much debate about Trump’s...}

\end{document}

%% file: packages_icml.tex
\usepackage{amsmath,amssymb,amsthm}

\newtheorem{theorem}{Theorem}

\theoremstyle{remark}
\newtheorem{remark}[theorem]{Remark}
\newtheorem{proposition}[theorem]{Proposition}
\newtheorem{lemma}[theorem]{Lemma}

\theoremstyle{definition}

%% file: intro.tex

\section{Introduction}

Large language models (LLMs) \citep{radford2018gpt1, radford2019language, brown2020gpt3} have revolutionized the realization of intelligence and significantly enhanced productivity across domains, including software development \citep{guo2025deepseek, gong2025diffucoder} and content generation \citep{shao2024deepseekmath}. LLMs adopt an autoregressive (AR) generation pipeline to predict the next token sequentially. Despite its simplicity, next-token generation becomes quite time-consuming for long sequences. Alternatively, motivated by the success of BERT \citep{devlin2019bert}, masked diffusion language models (MDLMs) \cite{austin2021structured, shi2024simplified, ou2024your, sahoo2024simple, gat2024discrete, campbell2022continuous, LLaDa1, ye2025dream7b}  propose a probabilistic, order-agnostic generation paradigm, which potentially leads to much reduced inference latency due to parallel generation and boosts the development of reasoning \citep{song2025seed, khanna2025mercury} and agentic artificial intelligence (AI).

Despite their potential for fast parallel generation, mainstream MDLMs typically achieve only comparable performance to AR by decoding a limited number of tokens per generation step using non-trivial sampling techniques. Decoding more tokens often disrupts inherent token dependencies and degrades reasoning performance. To address this issue, various acceleration strategies have been proposed, such as semi-autoregressive generation with block diffusion \cite{arriola2025block}, block-level KV cache \citep{ma2025dkvcache, wu2025fastdllm}, and distillation-based models \citep{chen2025dultra, Chen2025dParallel, qian2026d3llm, zheng2026ultrafast}. However, these approaches inevitably increase model complexity, making them harder to maintain and scale.

To avoid these complexities, we focus on the simple and effective self-correction techniques for maximizing the parallel generation abilities \cite{jiang2025diffusion}. A notable example is the confidence- or entropy-based \textit{remasking} step used during parallel generation to rectify erroneous tokens and revise previous mistakes \citep{LLaDa1, wang2025remasking, song2025seed, Chen2025dParallel}. Rather than relying solely on the inherent generation ability of MDLMs, \citet{kim2025fine} further enhances self-correction by fine-tuning pre-trained MDLMs with minimal effort. However, due to the limited generalization ability of post-training \citep{ouyang2022training, SutskeverDwarkesh2025ScalingToResearch,yue2025does}, self-correction remains underexplored.

\begin{figure*}[t]
  \begin{center}
    \centerline{\includegraphics[width=2\columnwidth]{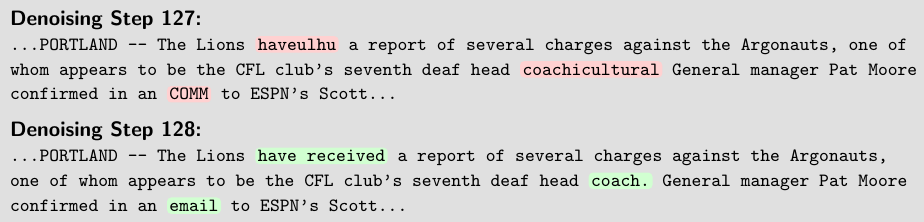}}
    \caption{
      Example of self-correction between two consecutive denoising steps (127→128). Generated by SCDD ($p_u=0.2$, trained on OWT) under 128 total denoising steps. Inappropriate tokens are directly corrected without remasking.
    }
    \label{fig:correction_example}
  \end{center}
  \vskip -0.2in
\end{figure*}
To improve the generalization of self-corrections, GIDD \cite{rutte2025generalized} studied the pretraining-based self-correction via a multi-step BERT-style objective that incorporates additional uniform transitions. However, its interpolation-based pipeline introduces opaque interactions between uniform transitions and absorbing masks, complicating hyperparameter tuning and hindering practical performance. Moreover, although GIDD enables token revision through additional uniform transitions, its reverse process still retains remasking behavior, which can make correction less direct and reduce the effective correction capacity under few-step parallel generation. To tackle these limitations, we propose a \textbf{S}elf-\textbf{C}orrecting \textbf{D}iscrete \textbf{D}iffusion (SCDD) model to reformulate pretrained self-correction with \textbf{clear} and \textbf{explicit} state transitions in discrete time. Our framework also simplifies the training noise schedule, eliminates a redundant \textit{remasking} step, and relies exclusively on uniform transitions to learn self-correction. To summarize, our contributions are three-fold: 
\begin{itemize}
    \item We redesign the forward process with Signal-to-Noise ratio (SNR) - informed parameters, thus providing {separate control} over different types of forward noising rates while maintaining clarity in marginal distribution representation {at the same time}.   
    \item Our training and inference pipeline is clean and engineering-light. More specifically, during training, the model is trained on theoretical ELBO loss without additional re-weighting; during inference, SCDD requires no post-hoc heuristic samplers and no hyperparameter-tuning. All token generation and correction are performed solely by running the backward process derived from Bayes’ rule.
    \item To the best of our knowledge, we are the first to train a diffusion language model that achieves self-correction completely free of \textit{remasking} during generation. Experiments conducted at the GPT-2 scale indicate that our model consistently surpasses existing benchmarks over a few standard datasets, achieving lower generative perplexity without sacrificing sample diversity in parallel generation settings. 
\end{itemize}

\begin{table*}[t]
\centering
\caption{Comparison of MDLM, GIDD, and SCDD. SCDD decouples token correction from masking, leading to a cleaner generator and explicit backward dynamics.}
\midsize
\setlength{\tabcolsep}{2pt}
\renewcommand{\arraystretch}{1.25}
\begin{tabular}{p{0.1\columnwidth}|p{0.85\columnwidth}|c|c|c|c}
\toprule
Model 
& Generator $R_t(\mathbf z_t,\mathbf z_s)$, $\mathbf z_s\neq\mathbf m$
& Self-Correction
& Remask-Free
& Closed-form Backward
& Decoupled SNRs \\
\midrule

MDLM
&
$\displaystyle
\frac{\gamma_t'}{\gamma_t}\,
\mathbf z_t^\top(\mathbf z_s-\mathbf m)$
&
\textcolor{red}{$\times$}
&
\textcolor{red}{$\times$}
&
\textcolor{blue}{\checkmark}
&
--
\\[0.8em]
\midrule

GIDD
&
$\displaystyle
\begin{aligned}
&
\left(
\frac{\gamma_t'}{\gamma_t}
+
\frac{\rho_t'}{\rho_t}
\right)
\mathbf z_s^\top\mathbf z_t
-
\mathbf z_t^\top
\Bigg[
\textcolor{red}{\gamma_t\frac{\rho_t'}{\rho_t}}\mathbf u+
\left(
\textcolor{red}{(1-\gamma_t)\frac{\rho_t'}{\rho_t}}
+
\frac{\gamma_t'}{\gamma_t}
\right)\mathbf m
\Bigg]
\end{aligned}$
&
\textcolor{blue}{\checkmark}
&
\textcolor{red}{$\times$}
&
\textcolor{red}{$\times$}
&
\textcolor{red}{$\times$}
\\[1.6em]
\midrule

SCDD
&
$\displaystyle
\begin{aligned}
&
\left(
\frac{\gamma_t'}{\gamma_t}
+
\frac{\rho_t'}{\rho_t}
\right)
\mathbf z_s^\top \mathbf z_t
-
\mathbf z_t^\top
\left(
\textcolor{blue}{\frac{\rho_t'}{\rho_t}}\mathbf u
+
\frac{\gamma_t'}{\gamma_t}\mathbf m
\right)
\end{aligned}$
&
\textcolor{blue}{\checkmark}
&
\textcolor{blue}{\checkmark}
&
\textcolor{blue}{\checkmark}
&
\textcolor{blue}{\checkmark}
\\[1.3em]

\bottomrule 
\end{tabular}
\label{tab:generator_comparison}
\end{table*}